%% file: arxiv.tex
\documentclass[11pt]{article}
\def\TSPREVIEW{1}
\def\ARXIVCOMBINED{1}

\input{shared}
\date{}

\ifdefined\TSPREVIEW
  \newenvironment{manuscriptwidefigure}{\begin{figure}[H]}{\end{figure}}
\else
  \newenvironment{manuscriptwidefigure}{\begin{figure*}[!t]}{\end{figure*}}
\fi

\ifpdf
\hypersetup{
  hypertexnames=false,
  pdftitle={Physical-Support Confidence Sets for Highly Coherent Dictionaries},
  pdfauthor={Guan-Ju Peng},
  pdfsubject={Physical-support confidence sets under learned-dictionary uncertainty},
  pdfkeywords={dictionary learning, sparse representation, physical support, coherent atoms, uncertainty quantification, resolution control}
}
\fi

\begin{document}
\input{main_body}

\clearpage
\section*{Supplementary Material}

\setcounter{section}{0}
\setcounter{subsection}{0}
\setcounter{subsubsection}{0}
\setcounter{equation}{0}
\setcounter{figure}{0}
\setcounter{table}{0}
\setcounter{algorithm}{0}
\setcounter{theorem}{0}
\renewcommand{\thesection}{S\arabic{section}}
\renewcommand{\thesubsection}{\thesection.\arabic{subsection}}
\renewcommand{\thesubsubsection}{\thesubsection.\arabic{subsubsection}}
\renewcommand{\theequation}{S\arabic{equation}}
\renewcommand{\thefigure}{S\arabic{figure}}
\renewcommand{\thetable}{S\arabic{table}}
\renewcommand{\thealgorithm}{S\arabic{algorithm}}

\input{supplement_body}
\end{document}

%% file: shared.tex
\usepackage[T1]{fontenc}
\usepackage[utf8]{inputenc}
\ifdefined\TSPPUBLICATION
\else
  \usepackage{lmodern}
\fi
\usepackage{microtype}
\usepackage{amsfonts,amssymb,mathtools,bm,mathrsfs}
\ifdefined\TSPPUBLICATION
\else
  \usepackage[a4paper,margin=0.92in]{geometry}
\fi
\usepackage{amsthm}
\usepackage{iftex,ifpdf}
\usepackage[caption=false,font=footnotesize]{subfig}
\usepackage[colorlinks=true,allcolors=blue!45!black]{hyperref}
\usepackage[capitalize,noabbrev]{cleveref}
\ifdefined\ARXIVCOMBINED
\else
  \usepackage{xr-hyper}
\fi
\usepackage{booktabs,array,longtable,tabularx}
\usepackage[numbers,sort&compress]{natbib}
\usepackage{graphicx}
\usepackage{float}
\usepackage{placeins}
\usepackage{xcolor}
\usepackage{enumitem}
\usepackage{tikz}
\usetikzlibrary{arrows.meta,calc,positioning,shapes.geometric}
\usepackage{algorithm}
\usepackage{algorithmic}

\ifpdf
  \DeclareGraphicsExtensions{.pdf,.png,.jpg,.eps}
\else
  \DeclareGraphicsExtensions{.eps}
\fi

\setlist[enumerate]{leftmargin=.42in}
\setlist[itemize]{leftmargin=.32in}

\newtheorem{theorem}{Theorem}[section]
\newtheorem{lemma}[theorem]{Lemma}
\newtheorem{corollary}[theorem]{Corollary}
\newtheorem{proposition}[theorem]{Proposition}

\theoremstyle{remark}
\newtheorem{remark}[theorem]{Remark}
\newtheorem{assumption}[theorem]{Assumption}
\ifdefined\TSPPUBLICATION
  \newenvironment{keywords}{\begin{IEEEkeywords}}{\end{IEEEkeywords}}
\else
  \newenvironment{keywords}{\par\smallskip\noindent\textbf{Index Terms---}}{\par\medskip}
\fi

\DeclareMathOperator*{\argmin}{arg\,min}

\renewcommand{\argmin}{\operatornamewithlimits{arg\,min}}

\ifdefined\RISKREDLINE
  
  \newenvironment{riskrepairblock}{\color{red!80!black}}{}
\else

\fi

\allowdisplaybreaks

\title{Physical-Support Confidence Sets for Highly Coherent Dictionaries}
\author{Guan-Ju Peng
\thanks{Institute of Data Science and Information Computing,
National Chung Hsing University, Taichung, Taiwan
(gjpeng@email.nchu.edu.tw).
ORCID: \url{https://orcid.org/0000-0001-5508-9485}.
Supplementary material and reproducibility code are also available at
\url{https://github.com/GJPengAtNchu/PhysicalSupportConfidenceSets}.}
}

%% file: main_body.tex
\maketitle

\input{sections/00_abstract}

\begin{keywords}
dictionary learning, sparse representation, physical-support inference,
highly coherent dictionaries, uncertainty quantification, resolution control
\end{keywords}

\input{sections/01_introduction}
\input{sections/02_model}
\input{sections/03_method}
\input{sections/04_theory_to_aeb}
\input{sections/05_experiments}
\FloatBarrier
\input{sections/06_conclusion}




\ifdefined\TSPPUBLICATION
\bibliographystyle{IEEEtranN}
\else
\bibliographystyle{unsrtnat}
\fi
\begingroup
\small
\raggedright
\bibliography{references}
\endgroup

%% file: sections/00_abstract.tex
\begin{abstract}
Sparse pursuit after dictionary learning can yield a precise atom support even
when its physical interpretation is not justified by the calibration data.
This problem is acute for highly coherent dictionaries, where alternative
calibration-compatible dictionaries may associate the same selected support
with different physical elements.

We develop resolution-aware physical-support inference that accounts jointly
for uncertainty in the learned dictionary and in the representation of a
deployment signal.  An exact cross-dictionary confidence correspondence
retains calibration-compatible dictionaries and deployment-compatible sparse
representations, then projects the surviving explanations onto
physical-support space. 
For local classes of coherent atom configurations with separation scale $s$,
once the deployment data resolve the coherent-block explanation and the atom
support within that block, we show that the minimax physical resolution from
$N$ calibration signals satisfies
\[
    \delta_{\mathrm{opt}}(N,s)
    \asymp
    \min\left\{
        s,\frac{1}{\sqrt{N}s^2}
    \right\}.
\]
Equivalently, resolution relative to the coherent-block scale is governed by
the orientation-information scale $Ns^6$.  Deployment replication can improve
physical localization only when orientation changes cannot be absorbed by
adjusting the active coefficients.

For computation, we introduce active endpoint bracketing (AEB), an adaptive
finite-bank procedure that evaluates only candidates that can still affect the
physical report.  It certifies a fine conclusion only when justified by all
unresolved outcomes; otherwise, it coarsens or abstains.
Finite-bank experiments, including a four-region synthetic application, show that the point-valued plug-in selector can report a finer physical interpretation than is supported by exhaustive uncertainty-aware evaluation of the same bank, while AEB avoids unsupported refinement with fewer candidate evaluations.
\end{abstract}

%% file: sections/01_introduction.tex
\section{Introduction}
\label{sec:introduction}

A common sparse-representation pipeline first learns a dictionary from
calibration data and then applies sparse pursuit to a deployment signal.
For calibration matrix $Y$, a canonical formulation is
\begin{align}
(\widehat D,\widehat C)
&\in
\arg\min_{D,C}\|Y-DC\|_F
\quad
\text{s.t.}\quad
\|c_i\|_0\le k,\ \|d_j\|_2=1,
\label{eq:intro-dictionary-learning}\\
\widehat x
&\in
\arg\min_x\|Z-\widehat D x\|_2
\quad
\text{s.t.}\quad
\|x\|_0\le k .
\label{eq:intro-sparse-pursuit}
\end{align}
The selected atoms are then interpreted as the physical elements contributing
to the deployment signal.  
Variants of this pipeline arise in array source localization, where atoms
represent candidate directions or locations; hyperspectral unmixing and Raman
spectroscopy, where atoms represent materials or chemical constituents; and
M/EEG source imaging, where atoms correspond to candidate cortical locations~\citep{MalioutovCetinWillsky2005,IordacheBioucasDiasPlaza2011,SunXin2014,GramfortKowalskiHamalainen2012}.
The coordinate support obtained under one fitted dictionary, however, need not
identify the physical support: highly coherent, calibration-compatible
dictionaries may attach different physical meanings to the same selected atom
index or support.

We address this gap through resolution-aware physical-support inference.
Figure~\ref{fig:intro_pipeline:a} summarizes the retain--project--coarsen
principle: retain dictionary--support--coefficient local explanations
compatible with calibration and deployment, project them through the physical
mapping, and report the finest physical conclusion shared by all survivors.
The result may identify individual elements, identify only a coherent group,
or retain support ambiguity. We formalize this construction as an exact
cross-dictionary confidence correspondence.

The general principle permits several coherent blocks and simultaneously
active components. The continuous theory isolates one supplied coherent block
and treats outside components as resolved and removable from the deployment
signal. This oracle-favorable reduction isolates ambiguity that persists even
when all other components are known.

Figure~\ref{fig:intro_pipeline:b} separates three progressively finer
distinctions: local explanation, atom support, and physical mapping. The first
two are deployment-side questions; the third is limited by dictionary
calibration uncertainty. Unresolved stages yield local-explanation ambiguity,
support ambiguity, or group-level rather than element-level physical support.

Once the deployment-side distinctions are resolved, the optimal physical
resolution for a coherent block with separation scale $s$ and $N$ calibration
signals is
\begin{equation}
\delta_{\mathrm{opt}}(N,s)
\asymp
s\wedge\frac{1}{\sqrt{N}s^2}.
\label{eq:intro_resolution_rate}
\end{equation}
Equivalently,
\begin{equation}
\frac{\delta_{\mathrm{opt}}(N,s)}{s}
\asymp
\min\left\{1,\frac{1}{\sqrt{Ns^6}}\right\},
\label{eq:intro_relative_resolution}
\end{equation}
so $Ns^6$ governs whether calibration resolves physical structure within the
coherent block. The sixth-order scale comes from cubic orientation sensitivity
of the calibration law. Deployment replication adds orientation information
only through the coefficient-profiled task secant; equal coefficients can be
exactly deployment-invariant.

\begin{manuscriptwidefigure}
\centering
\subfloat[Retain--project--coarsen pipeline. Calibration and deployment data constrain joint dictionary--support--coefficient explanations, which are projected to physical-support space and coarsened to the finest conclusion shared by all survivors.
\label{fig:intro_pipeline:a}]{%
  \includegraphics[width=.92\textwidth]
  {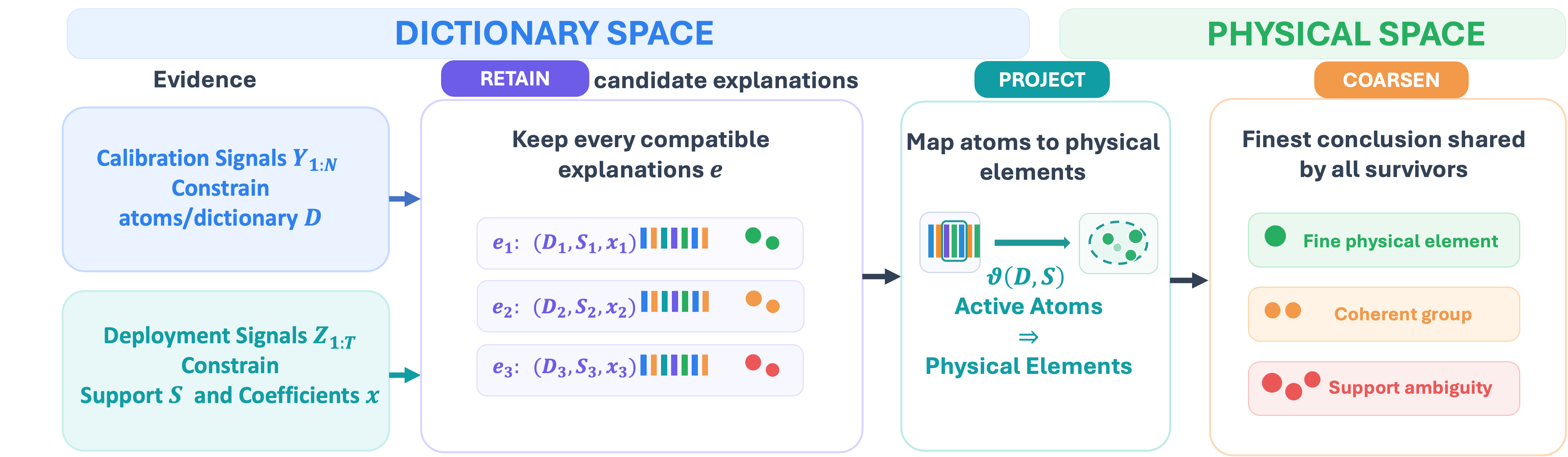}%
}

\medskip
\subfloat[Local resolution gates and reporting consequences.  Unresolved deployment information retains competing local explanations or atom supports; unresolved calibration mapping yields only a group-level conclusion; resolving all three gates supports element-level identification within the supplied coherent block.\label{fig:intro_pipeline:b}]{%
  \includegraphics[width=.92\textwidth]
  {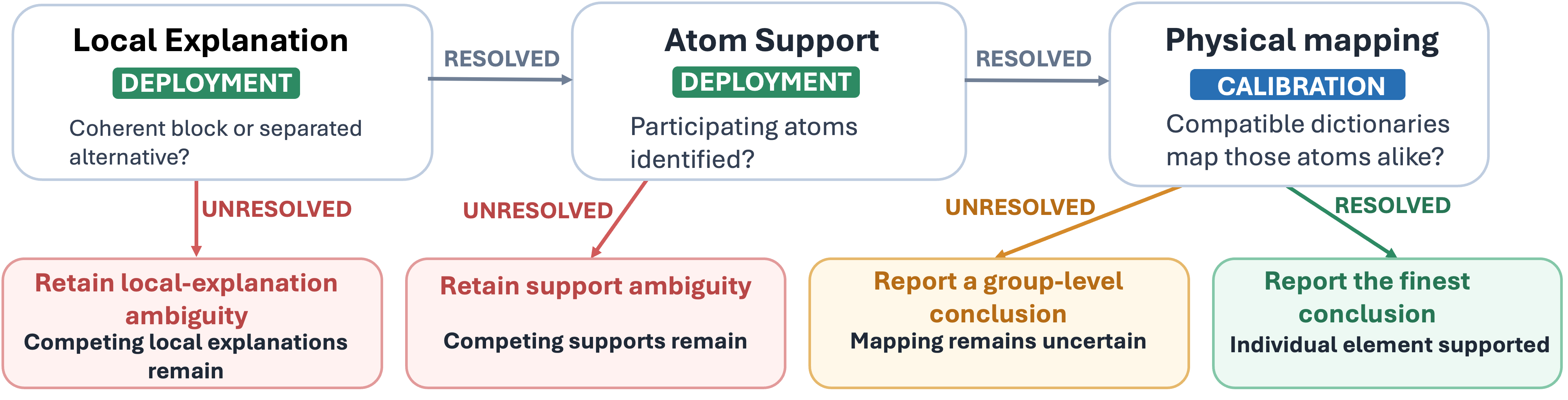}%
}

\caption{Resolution-aware physical-support inference after dictionary
learning. (a) illustrates the retain--project--coarsen principle for
joint dictionary--support--coefficient explanations. (b) specializes
the framework to one supplied coherent block and shows how unresolved
deployment and calibration uncertainty limit the attainable physical
resolution.}
\label{fig:intro_pipeline}
\end{manuscriptwidefigure}

The continuous correspondence supplies the local statistical benchmark. For
computation, active endpoint bracketing (AEB) operates on a specified finite
bank that may contain several simultaneously active components, adaptively
queries certificate-relevant candidates, and certifies a fine, coarser, or
ambiguous same-bank report, otherwise returning \textsc{abstain}. Its
statistical validity is on-bank; the finite bank is not assumed to approximate
or outer-cover the continuous parameter space.

The main contributions are:
\begin{enumerate}
\item We formulate physical-support uncertainty after dictionary learning and
construct an exact cross-dictionary confidence correspondence with conditional
coverage in the stated fixed-dimensional model.

\item We derive the minimax law in
Equation~\eqref{eq:intro_resolution_rate}, identify $Ns^6$ as the calibration
orientation-information scale, and characterize deployment information through
the coefficient-profiled task secant.

\item We develop AEB, an adaptive finite-bank method with deterministic
same-bank certification and on-bank statistical validity that safely coarsens
or abstains when a finer report is not certified.

\item Finite-bank studies show that point-valued plug-in selection can be
physically overprecise, whereas AEB recovers justified conclusions with
decision-dependent query savings.
\end{enumerate}

Related work falls into three categories:
\begin{itemize}
\item Dictionary-learning theory studies identifiability, error, sample
complexity, and minimax recovery of the dictionary
\citep{GribonvalSchnass2010,WuYu2018,GribonvalJenattonBach2015,JungEldarGoertz2016};
recent work also treats uncertainty for learned sparse estimators
\citep{HoppeEtAl2023,HoppeEtAl2024}. We instead propagate dictionary
uncertainty to the physical support of a separate deployment signal.

\item Group and hierarchical inference retains coarser conclusions under
collinearity \citep{Meinshausen2015,MandozziBuhlmann2016}, while sparse and
model confidence sets retain compatible supports
\citep{NicklVandeGeer2013,LiEtAl2019,LewisBattey2025}. These methods generally
treat the design as known rather than propagating calibration uncertainty.

\item Orbit estimation, singular models, weak identification, and set
identification provide related geometric and inferential perspectives
\citep{BandeiraEtAl2023,HoNguyen2019,Kaji2021,ChernozhukovHongTamer2007}.
Here the weakly identified geometry determines physical-support resolution for
a deployment signal.
\end{itemize}

%% file: sections/02_model.tex
\section{Model and Cross-Dictionary Physical-Support Inference}
\label{sec:problem}

We consider a fixed-dimensional train--test problem in which calibration data
constrain a dictionary and independent deployment data constrain a sparse
representation; both uncertainties are retained before projection to physical
support.

\subsection{Local calibration and deployment model}

We analyze one supplied coherent block of $q$ atoms and a well-separated
competing atom. For fixed $q\ge4$ and $n=q+1$, write
$
    D=(d_1,\ldots,d_q,a)\in\mathbb R^{q\times n},
$
with unit-norm columns, coherent block $d_1,\ldots,d_q$, and separated atom
$a$. Matching block and ambient dimensions gives a minimal local coordinate
representation and does not restrict the full dictionary size.

We use the star to denote the true parameter value.
The calibration sample consists of $N$ independent latent sparse mixtures perturbed by Gaussian noise:
\begin{equation}
    Y_i=D_\star c_i+\xi_i,
    \qquad
    \xi_i\sim\mathcal N_q(0,\nu_\star I_q),
    \qquad
    i=1,\ldots, N,
    \label{eq:training-model}
\end{equation}
where neither $c_i$ nor its support is observed and
$\nu_\star\in[\nu_-,\nu_+]\Subset(0,\infty)$. The calibration coefficients
follow the Bernoulli--Gaussian model
\begin{equation}
    c_{ij}=I_{ij}G_{ij},
    \qquad
    I_{ij}\stackrel{\mathrm{iid}}{\sim}
    \operatorname{Bernoulli}(p_\star),
    \qquad
    G_{ij}\stackrel{\mathrm{iid}}{\sim}\mathcal N(0,1),
    \label{eq:bg-code}
\end{equation}
with $p_\star\in[p_-,p_+]\Subset(0,1/2)$.

Components outside the supplied local subproblem are treated as resolved.
Independently, $T$ deployment measurements share the same sparse mean:
\begin{equation}
\begin{gathered}
    Z_\ell\stackrel{\mathrm{iid}}{\sim}
    \mathcal N_q(\mu_\star,\sigma_\star^2I_q),
    \qquad
    \ell=1,\ldots,T,\\
    \sigma_\star\in[\sigma_-,\sigma_+]\Subset(0,\infty).    
\end{gathered}
    \label{eq:test-model}
\end{equation}
Write $[q]=\{1,\ldots,q\}$ and fix a support size $2\le r\le q-1$.
Under the coherent-block explanation, the deployment mean is
\begin{equation}
\begin{gathered}
    \mu_\star
    =
    \sum_{j\in S_\star}x_{\star j}d_{\star j},
    \qquad
    S_\star\in\binom{[q]}r,\\
    x_{\star j}\in[\beta_-,\beta_+],
    \quad
    \beta_->0.
\end{gathered}
    \label{eq:fine-mean}
\end{equation}
The competing separated explanation is
\begin{equation}
    \mu_\star=\gamma_\star a_\star,
    \qquad
    \gamma_\star\in
    \Gamma_A=[\gamma_-,\gamma_+]\Subset(0,\infty),
    \label{eq:anchor-mean}
\end{equation}
denoted by $S_\star=\partial$. Deployment profiling first distinguishes this
separated alternative from the coherent-block explanation and then identifies
the atom support within the block; the remaining uncertainty is its physical
mapping.

\subsection{Coherence scale and physical-support target}

Fix a unit reference direction $u\in\mathbb R^q$, let $U=u^\perp$, and let
$v_1,\ldots,v_q\in U$ be a fixed centered regular-simplex template. Its exact
normalization identities are given in Supplementary
Section~\ref{sec:supp-local-coordinate-details}.

For $z\in U$ with $\|z\|<1$, define the unit-norm atom near $u$ by
\begin{equation}
    d(z)=\sqrt{1-\|z\|^2}\,u+z.
    \label{eq:atom-chart}
\end{equation}
An affine transformation of this template parameterizes the coherent block:
\begin{equation}
    z_j=b+Lv_j,
    \qquad
    d_j=d(z_j),
    \qquad
    j=1,\ldots,q,
    \label{eq:affine-simplex}
\end{equation}
where $b\in U$ is the center and invertible $L:U\to U$ controls scale,
orientation, and anisotropy. For supplied separation scale
$s\in(0,s_0]$, with $C_0s_0<1$, assume
\begin{equation}
\begin{gathered}
    \max_j\|b+Lv_j\|\le C_0s,
    \\
    \kappa_-s
    \le
    \sigma_{\min}(L)
    \le
    \sigma_{\max}(L)
    \le
    \kappa_+s.
\end{gathered}
    \label{eq:shell}
\end{equation}
These conditions keep the block in an $O(s)$ neighborhood of $u$ without directional
collapse.

The simplex geometry, shell bounds, and local chart imply
\begin{equation}
    \|d_i-d_j\|_2\asymp s,
    \qquad
    1-d_i^\top d_j
    =
    \frac12\|d_i-d_j\|_2^2
    \asymp s^2,
    \qquad i\ne j,
    \label{eq:s-coherence}
\end{equation}
so smaller $s$ means greater coherence.

For a reference $a_0$ separated from $u$ by a fixed positive angle, let
$P_a=aa^\top$ and $P_{a_0}=a_0a_0^\top$ and require
\begin{equation}
    \|P_a-P_{a_0}\|_\mathrm{F}\le C_as,
    \qquad
    0<\nu_-\le\nu\le\nu_+<\infty.
    \label{eq:anchor-noise-shell}
\end{equation}

Atom sign is physically irrelevant, so a unit atom represents the projective
element $[d]=\operatorname{span}\{d\}\in\mathbb{RP}^{q-1}$. Coherent-atom
ordering is also irrelevant, so $[D]$ denotes the orbit under simultaneous
permutation. The sign convention and quotient metric are given in
Supplementary Section~\ref{sec:supp-local-coordinate-details};
$\mathcal Q^D_{q,s}$ denotes the resulting dictionary shell.

The coordinate support $S$ indexes atoms in a particular dictionary,
whereas our inferential target is the \emph{physical support}: the set of physical elements represented by the active atoms.
For a coherent-block support $S$, the corresponding physical-support target is
\begin{equation}
    \vartheta(D,S)
    =
    \left(1,\{[d_j]:j\in S\}\right),
    \qquad
    \vartheta(D,\partial)
    =
    \left(0,\{[a]\}\right).
    \label{eq:physical-target}
\end{equation}
Write $\vartheta_\star=\vartheta(D_\star,S_\star)$. The binary component
marks the local explanation and the set component is its physical support.
The target takes values in
\begin{equation}
    \mathfrak V_q
    =
    \{0,1\}_{\mathrm{disc}}
    \times
    \mathcal K(\mathbb{RP}^{q-1}),
    \label{eq:target-space-main}
\end{equation}
where $\mathcal K(E)$ denotes the nonempty compact subsets of $E$; the full
marked metric is given in Supplementary
Section~\ref{sec:supp-local-coordinate-details}.

For unit atoms $d,d'$, use the sign-invariant projective distance
\begin{equation}
    d_{\mathrm{pr}}([d],[d'])
    =
    \sqrt{1-(d^\top d')^2},
    \label{eq:projective-distance}
\end{equation}
and let $d_H^{\mathrm{pr}}$ be its induced Hausdorff distance. For a retained
set $C$ of targets $(m,A)$, define the physical diameter
\begin{equation}
    \operatorname{diam}_{\mathrm{pr}}(C)
    =
    \sup_{(m,A),(m',A')\in C}d_H^{\mathrm{pr}}(A,A').
    \label{eq:physical-diameter}
\end{equation}
The full marked target-space metric is recorded in Supplementary
Section~\ref{sec:supp-local-coordinate-details}.

\begin{assumption}[Fixed-dimensional coherent train--test model]
\label{ass:frozen-shell}
The calibration and deployment models satisfy
Equations~\eqref{eq:training-model}--\eqref{eq:anchor-mean}; the coherent block, separated atom, and calibration-noise variance satisfy Equations~\eqref{eq:shell}
and~\eqref{eq:anchor-noise-shell}.
The calibration coefficient law is the Bernoulli--Gaussian model in
Equation~\eqref{eq:bg-code}.  Deployment coefficients satisfy
$x_j\in[\beta_-,\beta_+]$ with $\beta_->0$, and
$\sigma\in[\sigma_-,\sigma_+]\Subset(0,\infty)$.
The ambient dimension $q$, support size $r$, compact
nuisance ranges, and shell constants are fixed.
The scale $s\in(0,s_0]$ is supplied and is not estimated from the data.
\end{assumption}

Let $\mathfrak P^p_{q,r}(s)$ denote the resulting joint parameter class, including both local deployment explanations.
This oracle-favorable local formulation isolates physical-resolution loss due
to uncertainty in the coherent dictionary geometry.

\subsection{Cross-dictionary confidence correspondence}
Fix calibration and deployment error levels $\alpha_D,\alpha_T\in(0,1)$.
For calibration parameter $\theta=(p,\nu,[D])$, use the estimable second
moment and fourth cumulant
\begin{equation}
    M_2(\theta)=\mathbb E_\theta(YY^\top),
    \qquad
    K_4(\theta)=\operatorname{cum}_{4,\theta}(Y).
    \label{eq:training-moments}
\end{equation}
which constrain the local dictionary geometry under Bernoulli--Gaussian
coding.

\begin{lemma}[Population moment identifiability]
\label{lem:moment-reduction}
Under Bernoulli--Gaussian coding with $p<1/2$, the population pair
$(M_2,K_4)$ identifies $p$, $\nu$, and the aggregate second- and fourth-order
dictionary coordinates used below.
\end{lemma}

The notation, inversion formulas, and proof are given in Supplementary
Sections~\ref{sec:supp-moment-identification}
and~\ref{sec:supp-proof-training-geometry}. Let robust estimators
$\widehat M_2,\widehat K_4$ and radii $\epsilon_{2,N},\epsilon_{K,N}$ satisfy
\begin{equation}
\begin{aligned}
&
\inf_{p,\nu,D}
\mathbb P_{p,\nu,D}^{N}
\left\{
    \|\widehat M_2-M_2\|_F\le\epsilon_{2,N},
    \;
    \|\widehat K_4-K_4\|_F\le\epsilon_{K,N}
\right\}
\\[-1mm]
&\hspace{18mm}\ge1-\alpha_D.
\end{aligned}
\label{eq:training-rectangle}
\end{equation}
Supplementary Section~\ref{sec:supp-mom-region} gives an explicit
coordinatewise median-of-means construction with $O(N^{-1/2})$ radii above a
fixed-dimensional threshold and retains the full compact shell below it.
The calibration-compatible region is
\begin{equation}
\widehat{\mathcal K}_{q,s}^{p}
=
\left\{
\begin{aligned}
    &(p,\nu,[D]):\\
    p\in[p_-,p_+],\;
    \nu\in[\nu_-,\nu_+],\;
    [D]&\in\mathcal Q^D_{q,s},\\
    \|\widehat M_2-M_2(p,\nu,D)\|_F
    &\le\epsilon_{2,N},\\
    \|\widehat K_4-K_4(p,\nu,D)\|_F
    &\le\epsilon_{K,N}
\end{aligned}
\right\}.
\label{eq:training-profile}
\end{equation}
On the event in Equation~\eqref{eq:training-rectangle}, it contains
$(p_\star,\nu_\star,[D_\star])$.

For deployment, define the sample mean and common confidence radius
\begin{equation}
    \bar Z
    =
    \frac{1}{T}\sum_{\ell=1}^{T}Z_\ell,
    \qquad
    \tau_{T,q}
    =
    \frac{\sigma_+}{\sqrt T}
    \sqrt{\chi^2_{q,1-\alpha_T}},
    \label{eq:test-radius}
\end{equation}
where $\chi^2_{q,1-\alpha_T}$ is the indicated $\chi_q^2$ quantile. Uniformly
over the allowed noise range, the event
\begin{equation}
    \|\bar Z-\mu_\star\|_2\le\tau_{T,q}
    \label{eq:test-truth-event}
\end{equation}
has probability at least $1-\alpha_T$.

For retained $D$, profile the deployment residual over unknown coefficients:
\begin{equation}
    \ell_{\mathrm{prof}}(\bar Z,D,S)
    =
    \min_{(x_j)_{j\in S}\in[\beta_-,\beta_+]^r}
    \left\|
        \bar Z-\sum_{j\in S}x_jd_j
    \right\|_2,
    \label{eq:fine-profile-loss}
\end{equation}
and for the separated alternative,
\begin{equation}
    \ell_{\mathrm{prof}}(\bar Z,D,\partial)
    =
    \min_{\gamma\in\Gamma_A}
    \|\bar Z-\gamma a\|_2.
    \label{eq:anchor-profile-loss}
\end{equation}
Interpreting $(D,S)$ modulo simultaneous coherent-atom permutation, combine
calibration and deployment compatibility in
\begin{equation}
\mathcal R_{q,r,s}^{p}
=
\left\{
(p,\nu,[D,S]):
\begin{array}{l}
    (p,\nu,[D])\in\widehat{\mathcal K}_{q,s}^{p},\\
    S\in\binom{[q]}r\cup\{\partial\},\\
    \ell_{\mathrm{prof}}(\bar Z,D,S)\le\tau_{T,q}
\end{array}
\right\}.
\label{eq:raw-profile}
\end{equation}
All candidates use the same confidence ball for the unknown deployment mean,
so this truth-retention guarantee needs no multiplicity correction over
dictionaries or supports.

Projecting the joint feasible explanations through $\vartheta$ gives
\begin{equation}
\widehat{\mathfrak C}_{q,r,s}^{p}
=
\left\{
    \vartheta(D,S):
    \exists\,p,\nu\text{ such that }
    (p,\nu,[D,S])\in\mathcal R_{q,r,s}^{p}
\right\},
\label{eq:reported-set}
\end{equation}
when $\mathcal R_{q,r,s}^{p}\ne\varnothing$; record
$F_{\mathrm{empty}}=\mathbf1\{\mathcal R_{q,r,s}^{p}=\varnothing\}$ otherwise.
Formal totalization and measurability are given in Supplementary
Section~\ref{sec:supp-measurability}.

\begin{proposition}[Truth retention of the cross-dictionary confidence correspondence]
\label{prop:profile-well-defined}
On the intersection of the calibration event in
Equation~\eqref{eq:training-rectangle} and the deployment event in~\eqref{eq:test-truth-event}, the true explanation satisfies $
    (p_\star,\nu_\star,[D_\star,S_\star])
    \in
    \mathcal R_{q,r,s}^{p}
$.
Consequently, we have $
    \vartheta_\star
    \in
    \widehat{\mathfrak C}_{q,r,s}^{p},
    F_{\mathrm{empty}}=0
$.
\end{proposition}

The proof and regularity details are in Supplementary
Sections~\ref{sec:supp-measurability}
and~\ref{sec:supp-proof-honest-region}. If both local explanations survive,
the local explanation is unresolved; if only the coherent-block explanation
survives with several supports, atom support is ambiguous. A single surviving
support can still have uncertain physical mapping across
calibration-compatible dictionaries.

We use this continuous correspondence as the statistical benchmark and
develop the finite-bank computational procedure separately in
Section~\ref{sec:finite-bank-controller}.

%% file: sections/03_method.tex
\section{Coherence, Information, and Minimax Physical Resolution}
\label{sec:results}

We now ask how accurately the physical support can be localized as the coherent atoms become increasingly similar.  
The analysis has three steps. First, we determine how the calibration information about coherent-block orientation scales with coherence.  
Then, we translate this information into the physical resolution of the cross-dictionary correspondence and establish a matching minimax lower bound.  
Finally, we determine when additional deployment measurements can improve this calibration-limited resolution.

\subsection{Calibration information under high coherence}

The main calibration bottleneck is that different orientations of a highly coherent block can induce nearly indistinguishable calibration distributions. To identify where orientation first appears, we use a third-order tensor representation.  
For a vector $v$, let $v^{\otimes3}=v\otimes v\otimes v$ denote the rank-one third-order tensor with entries $(v^{\otimes3})_{abc}=v_av_bv_c$.  We then define
\begin{equation}
    T_q=\sum_{j=1}^q v_j^{\otimes3},
    \qquad
    G_2=LL^\top,
    \qquad
    G_3=L^{\otimes3}T_q,
    \label{eq:result-invariants}
\end{equation}
where $L^{\otimes3}=L\otimes L\otimes L$ applies $L$ along each tensor mode.  
Here, $G_2$ describes the second-order shape of the local block, whereas $G_3$ is the first permutation-invariant coordinate that retains its residual orientation.
Because $L$ has scale $s$, $G_3$ has scale $s^3$.

To isolate this orientation effect, we fix a unit vector $e_1\in U$ and a constant $\lambda_\star\in(\kappa_-,\kappa_+)\cap(0, C_0)$, and consider the centered balanced submodel:
\begin{equation}
\begin{gathered}
    a_0=(u+e_1)/\sqrt2,
    \qquad
    b=0,
    \\
    L=\lambda_\star sR,
    \qquad
    R\in O(U),
    \qquad
    P_a=P_{a_0}.
\end{gathered}
    \label{eq:balanced-hard-core}
\end{equation}
where $O(U)$ is the orthogonal group on $U$.  In this submodel, only the orientation $R$ varies, while the remaining local geometry is left unaltered.

The cubic scaling of $G_3$ identifies where residual orientation enters the local geometry, but it does not yet determine how strongly that orientation is visible in the calibration data.  
To quantify statistical distinguishability, we therefore perturb the block orientation along a local path $R_t=R\exp(t\Omega)$ and examine the resulting first-order change in the calibration density. The next result shows that this change in density is itself of order $s^3$.  
Consequently, the corresponding Fisher information is of order $s^6$, which is the calibration information scale governing the physical resolution derived below.

Let $f_{s,R}$ denote the density of one calibration observation on this
submodel, with all remaining calibration parameters held fixed at compact
interior values, and let $f_0$ denote the collapsed $s=0$ calibration model.

\begin{theorem}[Calibration information for coherent-block orientation]
\label{thm:training-geometry}
Under Assumption~\ref{ass:frozen-shell}, there exist constants
$0<c<C<\infty$ such that the following statements hold.
For $R_t=R\exp(t\Omega)$ with $\Omega^\top=-\Omega$,
\begin{equation}
\left.
\frac{\partial}{\partial t}
 f_{s,R_t}
\right|_{t=0}
=
 s^3g_{R,\Omega}
+
O_{L^2(f_0^{-1})}
\left(s^4\|\Omega\|_\mathrm{F}\right),
\label{eq:cubic-density-score}
\end{equation}
where $g_{R,\Omega}\neq0$ for every nonzero orientation tangent modulo
coherent-atom permutation.

More generally, let $I_R^{\mathrm{eff}}(\Omega;\theta)$ denote the efficient
Fisher information for an orientation perturbation after profiling the regular
nuisance parameters at calibration parameter $\theta$. Uniformly over compact
interior submodels,
\begin{equation}
    cs^6\|\Omega\|_\mathrm{F}^2
    \le
    I_R^{\mathrm{eff}}(\Omega;\theta)
    \le
    Cs^6\|\Omega\|_\mathrm{F}^2.
    \label{eq:profiled-orientation-information}
\end{equation}
\end{theorem}

Thus one calibration observation carries only order-$s^6$ information about
the orientation of the coherent block. With $N$ independent calibration
signals, the relevant orientation-information scale is therefore
\begin{equation}
    I_D=Ns^6.
    \label{eq:calibration-information-scale}
\end{equation}
The all-pair statistical chord bounds and the inverse map from invariant
coordinates to physical dictionary geometry used in the proofs are given in
Supplementary Sections~\ref{sec:supp-invariant-coordinate-bounds}
and~\ref{sec:supp-simplex-quotient}.

Orientation uncertainty is only one source of physical imprecision.
Even if the dictionary has been localized to a small neighborhood, the deployment data may still be compatible with different atom supports or even with different local explanations.  
To determine when these ambiguities can be resolved, we need quantitative
lower bounds on how far apart competing deployment means remain after allowing
the dictionary to vary within quotient-dictionary radius $\rho$.

The next lemma provides these separation margins.  
The quantity $m_S(\rho)$ controls the separation between distinct coherent-block supports, whereas $m_G(\rho)$ controls the separation between a coherent-block explanation and the separated alternative.
These margins will be compared with the deployment uncertainty radius in the coverage and resolution bound below.
\begin{lemma}[Support and local-explanation separation across nearby dictionaries]
\label{lem:cross-dictionary-margins}
There exist fixed constants $c_D,g_G^0,C_S,C_G>0$ such that every dictionary
in the supplied local shell satisfies
\begin{equation}
    \left\|
        \sum_{j=1}^qh_jd_j
    \right\|_2
    \ge
    c_Ds\|h\|_2,
    \qquad
    h\in\mathbb R^q.
    \label{eq:lower-singular-margin}
\end{equation}
For $t\in\mathbb R$, write $[t]_+=\max\{t,0\}$.
If two candidate dictionaries have quotient-dictionary distance at most
$\rho$, then two distinct size-$r$ coherent-block supports are separated by
at least
\begin{equation}
    m_S(\rho)
    =
    \left[
        \sqrt{2}\,c_D\beta_-s
        -
        C_Sr\beta_+\rho
    \right]_+,
    \label{eq:support-margin}
\end{equation}
whereas a coherent-block mean and a separated-atom mean are separated by at
least
\begin{equation}
    m_G(\rho)
    =
    \left[
        g_G^0
        -
        C_G(r\beta_++\gamma_+)\rho
    \right]_+.
    \label{eq:parent-margin}
\end{equation}

The constants can be chosen so that $m_S(\rho)>0$ also guarantees a unique matching of the coherent atoms across the two dictionaries.
\end{lemma}

\subsection{Coverage and physical-resolution upper bound}

We now combine calibration uncertainty with the deployment separation margins to bound the physical diameter of the reported correspondence.  
For a joint parameter $\xi\in\mathfrak P^p_{q,r}(s)$, we write $\vartheta(\xi)$ for the physical-support target induced by its dictionary and deployment support.

Recall that the calibration-compatible region is controlled by the estimation errors of the second moment and fourth cumulant, with radii $\epsilon_{2, N}$ and $\epsilon_{K, N}$, respectively. The dictionary-radius bound below depends on these two calibration errors through their joint magnitude, so we define
\begin{equation}
    \epsilon_N
    =
    \left(
        \epsilon_{2,N}^2+\epsilon_{K,N}^2
    \right)^{1/2}.
    \label{eq:combined-moment-radius}
\end{equation}

The next result establishes the two properties required for a useful physical-support confidence correspondence.  
First, it verifies that the reported set retains the true physical support with the prescribed calibration and deployment error probabilities.  
Second, it converts the calibration uncertainty summarized by $\epsilon_N$, together with the deployment separation margins from Lemma~\ref{lem:cross-dictionary-margins}, into an explicit bound on the physical diameter of the correspondence.
\begin{theorem}[Coverage and resolution of the cross-dictionary correspondence]
\label{thm:honest-region}
Under Assumption~\ref{ass:frozen-shell} and the calibration coverage condition
in Equation~\eqref{eq:training-rectangle}, the totalized correspondence
$\widehat{\mathfrak C}_{q,r,s}^{p}$ is measurable, nonempty, and
compact-valued. Uniformly over $\mathfrak P^p_{q,r}(s)$,
\begin{equation}
\mathbb P_{p_\star,\nu_\star,D_\star}^N
\left[
\mathbb P_{\mu_\star,\sigma_\star}^T
\left\{
    \vartheta_\star
    \in
    \widehat{\mathfrak C}_{q,r,s}^{p}
    \mid Y_{1:N}
\right\}
\ge1-\alpha_T
\right]
\ge1-\alpha_D.
\label{eq:two-level-coverage}
\end{equation}
Consequently, by independence of the calibration and deployment experiments,
with
\[
    \alpha
    =
    1-(1-\alpha_D)(1-\alpha_T),
\]
\begin{equation}
\inf_{\xi\in\mathfrak P^p_{q,r}(s)}
\mathbb P_\xi^{N,T}
\left\{
    \vartheta(\xi)
    \in
    \widehat{\mathfrak C}_{q,r,s}^{p}
\right\}
\ge1-\alpha.
\label{eq:marginal-coverage}
\end{equation}

For $N$ above the fixed block threshold of the robust calibration-moment
construction, define
\begin{equation}
\rho_N(s)
=
C\left[
    (s\wedge\epsilon_N)
    +
    \left(s\wedge\frac{\epsilon_N}{s}\right)
    +
    \left(s\wedge\frac{\epsilon_N}{s^2}\right)
\right].
\label{eq:dictionary-radius}
\end{equation}
Any two dictionaries retained by the same nonempty calibration region have
quotient-dictionary distance at most $\rho_N(s)$.  If
$\epsilon_N\le CN^{-1/2}$, then
\begin{equation}
    \rho_N(s)
    \le
    C\left(
        s\wedge\frac{1}{\sqrt N\,s^2}
    \right).
    \label{eq:dictionary-radius-rate}
\end{equation}
Below the block threshold, retaining the full compact shell yields the same
order after adjustment of fixed constants.

For every realization,
\begin{align}
\operatorname{diam}_{\mathrm{pr}}
\left(
    \widehat{\mathfrak C}_{q,r,s}^{p}
\right)
\le{}&
C\rho_N(s)
\notag\\
&+
Cs\,
\mathbf 1
\left\{
    2\tau_{T,q}\ge m_S(\rho_N(s))
\right\}
\notag\\
&+
C\,
\mathbf 1
\left\{
    2\tau_{T,q}\ge m_G(\rho_N(s))
\right\}.
\label{eq:three-scale-upper-bound}
\end{align}
\end{theorem}

The three terms in $\rho_N(s)$ arise from uncertainty in the local center, second-order shape, and cubic orientation coordinate, respectively.  
In the high-coherence regime, the orientation term produces the characteristic calibration-limited rate $s\wedge(\sqrt N\,s^2)^{-1}$.

The diameter bound also separates the three sources of uncertainty introduced in Section~\ref{sec:problem}.  
The two deployment-side scales follow from a common Gaussian mean-separation principle.  
For $T$ independent observations with noise variance $\sigma^2$, two candidate means separated by distance $\Delta$ have KL divergence of order $T\Delta^2/\sigma^2$.  
By Lemma~\ref{lem:cross-dictionary-margins}, the coherent-block versus separated alternative has baseline mean separation of order $r\beta_-$, whereas two distinct coherent-block supports have separation of order $\beta_-s$.
Using the worst-case deployment noise level $\sigma_+$ therefore gives the first two information scales below.  
The third follows from Theorem~\ref{thm:training-geometry}, which gives order-$s^6$ orientation information per calibration signal:
\begin{equation}
    I_G^{(r)}
    =
    \frac{Tr^2\beta_-^2}{\sigma_+^2},
    \qquad
    I_S
    =
    \frac{T\beta_-^2s^2}{\sigma_+^2},
    \qquad
    I_D
    =
    Ns^6.
    \label{eq:theorem-three-gates}
\end{equation}
Here $I_G^{(r)}$ measures deployment information for distinguishing the coherent-block explanation from the separated alternative, $I_S$ measures deployment information for distinguishing supports within the coherent block, and $I_D$ measures calibration information about its physical orientation.

If the deployment data cannot distinguish the two local explanations, the
physical diameter can remain on the order of one. Once the coherent-block
explanation is identified, but its atom support remains unresolved, the
relevant scale is order $s$. After both deployment-side ambiguities are
resolved, the remaining physical uncertainty is governed by
\begin{equation}
    s\wedge\frac{1}{\sqrt N\,s^2}.
    \label{eq:resolved-upper-rate}
\end{equation}
These are distinct information requirements; the general analysis does not
collapse them into a single $NT$ information scale.

\subsection{Minimax lower bound and optimal physical resolution}

The upper bound above shows the physical resolution achieved by the proposed cross-dictionary correspondence, but it does not determine whether a different valid confidence procedure could do better.  
To distinguish a limitation of the procedure from an intrinsic statistical limitation, we now derive a minimax lower bound over all uniformly valid physical-support confidence correspondences.

To compare such procedures at the same coverage level, we assume
$(1-\alpha_D)(1-\alpha_T)>1/2$ and set
\[
    \alpha
    =
    1-(1-\alpha_D)(1-\alpha_T)
    <\frac12 .
\]
We define the class of uniformly valid correspondences by
\begin{equation}
\begin{aligned}
\mathfrak H_\alpha
\left\{
    \mathfrak P^p_{q,r}(s)
\right\}
&=
\left\{
\begin{aligned}
    \widehat C:
    \inf_{\xi\in\mathfrak P^p_{q,r}(s)}
    \\
    \mathbb P_\xi^{N,T}
    \left\{
        \vartheta(\xi)\in\widehat C
    \right\}
    \ge1-\alpha
\end{aligned}
\right\},
\end{aligned}
\label{eq:honest-class}
\end{equation}
where $\widehat C$ ranges over measurable nonempty compact-valued correspondences.  
Among all such valid procedures, we measure the best achievable worst-case physical resolution by the minimax expected diameter
\begin{equation}
\mathcal R_{N,T}^{(q,r)}(s)
=
\inf_{\widehat C\in\mathfrak H_\alpha}
\sup_{\xi\in\mathfrak P^p_{q,r}(s)}
\mathbb E_\xi
\operatorname{diam}_{\mathrm{pr}}(\widehat C).
\label{eq:minimax-diameter-risk}
\end{equation}

The lower bound uses three two-point comparisons corresponding to the three sources of physical uncertainty identified above: the local explanation, the atom support within the coherent block, and the physical orientation of the dictionary.
A mild interior compatibility condition needed for the coherent-block versus separated comparison, together with the generic two-point diameter argument, is given in Supplementary Section~\ref{sec:supp-minimax-setup}.

The next theorem shows that the three resolution scales appearing in the upper bound are unavoidable.  
In particular, once the two deployment-side ambiguities are resolved, the proposed correspondence attains the optimal calibration-limited physical resolution.

\begin{theorem}[Minimax physical-resolution lower bound and optimality]
\label{thm:fixed-shell-minimax}
Under Assumption~\ref{ass:frozen-shell} and the interior compatibility
condition in Supplementary Section~\ref{sec:supp-minimax-setup}, there exist
fixed positive constants and transition thresholds such that
\begin{equation}
\begin{aligned}
\mathcal R_{N,T}^{(q,r)}(s)
&\ge
c\max\left[
    \mathbf 1\{I_G^{(r)}\le c_G\},
    \;
    s\,\mathbf 1\{I_S\le c_S\},
    \right.\\[-1mm]
&\hspace{30mm}\left.
    s\wedge\frac{1}{\sqrt N\,s^2}
\right].
\end{aligned}
\label{eq:three-gate-lower-bound}
\end{equation}
When $I_G^{(r)}$ and $I_S$ exceed their fixed resolution thresholds, the cross-dictionary correspondence of Theorem~\ref{thm:honest-region}, with $\epsilon_N\le CN^{-1/2}$, satisfies 
\[
\sup_{\xi\in\mathfrak P^p_{q,r}(s)}
\mathbb E_\xi
\operatorname{diam}_{\mathrm{pr}}
\left(
    \widehat{\mathfrak C}_{q,r,s}^{p}
\right)
\le
C\left(
    s\wedge\frac{1}{\sqrt N\,s^2}
\right).
\]
Together with the lower bound, this yields 
\begin{equation}
    \mathcal R_{N,T}^{(q,r)}(s)
    \asymp
    s\wedge\frac{1}{\sqrt N\,s^2}.
    \label{eq:resolved-minimax-rate}
\end{equation}
\end{theorem}

Thus, once the deployment data resolve the local explanation and atom support, the calibration-limited physical resolution is minimax optimal up to fixed constants:
\begin{equation}
    \delta_{\mathrm{opt}}(N,s)
    \asymp
    s\wedge\frac{1}{\sqrt N\,s^2}.
    \label{eq:optimal-physical-resolution}
\end{equation}
Equivalently,
\begin{equation}
    \frac{\delta_{\mathrm{opt}}(N,s)}{s}
    \asymp
    \min\left\{
        1,\frac{1}{\sqrt{Ns^6}}
    \right\}.
    \label{eq:relative-physical-resolution}
\end{equation}
This is a constant-factor result for the supplied fixed-dimensional local class, and the coherence scale $s$ is treated as a given parameter. 

The minimax theorem characterizes the optimal physical-resolution rate, but two consequences are especially useful for interpreting what this limitation means for sparse inference.  
First, we compare the learned-dictionary problem with the idealized case in which the dictionary is known.  
This isolates the loss of physical resolution caused specifically by a finite
calibration sample.

\begin{corollary}[Known- versus learned-dictionary resolution]
\label{cor:known-learned-gap}
Consider a sequence of problems for which
\[
    I_G^{(r)}\rightarrow\infty,
    \qquad
    I_S\rightarrow\infty,
\]
while $I_D=Ns^6$ remains bounded. If the dictionary were known, the
deployment experiment would eventually separate the competing atom supports.
When the dictionary is learned from finite calibration data, however, every
uniformly valid physical-support confidence correspondence has worst-case
expected diameter at least $cs$. Hence a support decision that is precise
conditional on one fitted dictionary can be strictly more precise than the
physical conclusion justified after dictionary uncertainty is taken into
account.
\end{corollary}

The preceding comparison could still leave open whether this gap is caused by uncertainty in the deployment support, coefficients, noise levels, or other nuisance quantities rather than by dictionary orientation itself.  
To isolate the calibration bottleneck, we therefore consider an oracle-favorable experiment in which all such nonorientation quantities are revealed.

\begin{corollary}[Calibration bottleneck persists with oracle side information]
\label{cor:oracle-persistence}
On the balanced orientation submodel in Equation~\eqref{eq:balanced-hard-core},
suppose an oracle reveals the coherent-block identity, the scale $s$, the active
support, equal positive coefficients, the separated atom, the coefficient
law, the noise levels, and all nonorientation nuisance parameters. Over a
fixed injective local orientation path, every uniformly valid confidence
correspondence still satisfies
\begin{equation}
    \inf_{\widehat C}
    \sup_{|h|\le h_0}
    \mathbb E_h
    \operatorname{diam}_{\mathrm{pr}}(\widehat C)
    \ge
    c\left(
        s\wedge\frac{1}{\sqrt N\,s^2}
    \right),
    \label{eq:oracle-dictionary-lower-bound}
\end{equation}
where the infimum ranges over correspondences with marginal coverage at least
$1-\alpha$, for fixed $\alpha<1/2$.
\end{corollary}

Thus the calibration bottleneck is intrinsic to uncertainty in coherent-block
orientation: it persists even when the deployment support and all
nonorientation quantities are supplied. The proof is given in
Supplementary Section~\ref{sec:supp-proof-minimax}.

\begin{remark}[Embedding in a larger dictionary]
\label{rem:larger-dictionary-embedding}
The local hard core above can be embedded isometrically into a higher-dimensional signal space and included as a submodel of a larger dictionary.
All additional atoms and orthogonal coordinates may be held fixed while the coherent block follows the least-favorable orientation path.  
Consequently, the minimax lower bound applies to any larger dictionary model containing this local subexperiment.
\end{remark}

\subsection{When deployment replication adds orientation information}

The preceding results identify a calibration-limited physical-resolution barrier.  
A natural question is whether this barrier can be reduced simply by collecting more deployment measurements after the local explanation and atom support have been resolved.
The answer depends on whether a change in dictionary orientation leaves a component in the deployment mean that cannot be reproduced by changing the active coefficients.

To quantify this distinction, consider a centered coherent-block support $S$ and an interior coefficient vector  $x\in(\beta_-,\beta_+)^r$.  
An infinitesimal orientation perturbation changes the deployment mean through $\dot D_S[\Omega]x$, but part of this change may lie inside $\operatorname{span}(D_S)$. It can therefore be absorbed by an infinitesimal coefficient adjustment $D_S\dot x$.  
We measure only the component that remains after this optimal adjustment by the coefficient-profiled orientation secant
\begin{equation}
\begin{aligned}
\chi_{\mathrm{tan}}(D,S,x;\Omega)
&=
\frac{1}{s}
\inf_{\dot x\in\mathbb R^r}
\left\|
    \dot D_S[\Omega]x+D_S\dot x
\right\|_2
\\
&=
\frac{1}{s}
\left\|
    P_{\operatorname{span}(D_S)}^\perp
    \dot D_S[\Omega]x
\right\|_2,
\end{aligned}
\label{eq:tangent-secant}
\end{equation}
where $\dot D_S[\Omega]$ denotes the derivative of the active-atom matrix under the orientation tangent $\Omega$.  
Thus $\chi_{\mathrm{tan}}=0$ means that the local orientation change is invisible to the deployment experiment after
profiling the coefficients, whereas $\chi_{\mathrm{tan}}>0$ means that deployment data contain additional orientation information.

The next result converts this geometric quantity into statistical information. It answers two questions: how much orientation information is contributed by $T$ deployment replicates, and how that information combines with the calibration information $Ns^6$ when both sources are informative.

\begin{theorem}[Deployment information for physical orientation]
\label{thm:task-symmetry}
For every fixed $2\le r\le q-1$ and every interior coefficient vector $x$,
the efficient information contributed by $T$ deployment measurements for an
orientation tangent $\Omega$ is
\begin{equation}
    I_{\Omega\mid x}^{\mathrm{test}}
    =
    \frac{Ts^2}{\sigma^2}
    \chi_{\mathrm{tan}}^2(D,S,x;\Omega).
    \label{eq:test-efficient-information}
\end{equation}

Moreover, consider a compact orientation orbit and compact interior
coefficient set $\mathcal X_{\mathrm{orb}}$ for which physical target
separation is of order $s|\phi-\phi'|$ and the coefficient-profiled deployment
mean separation is of order $s\chi_s|\phi-\phi'|$.  Under the precise
regularity conditions stated in Supplementary
Section~\ref{sec:supp-task-orbit}, the minimax physical diameter on this
restricted orbit satisfies
\begin{equation}
\mathcal R_{N,T}^{\mathrm{task}}
(s,\mathcal X_{\mathrm{orb}},\sigma)
\asymp
s
\wedge
\frac{s}{
\sqrt{
    Ns^6
    +
    T\chi_s^2s^2/\sigma^2
}}.
\label{eq:restricted-task-rate}
\end{equation}
\end{theorem}

The first part of the theorem gives the general local message: deployment replication improves orientation resolution only through the component measured by $\chi_{\mathrm{tan}}$.  
When this component vanishes, increasing $T$ cannot overcome the calibration bottleneck. When it is nonzero, the deployment contribution adds to the calibration information.  
The second part makes this combination explicit on the declared restricted orientation orbit.

The secant formulation is general but somewhat abstract. To make its operational meaning transparent, we next specialize to two active coherent atoms.
In this case, the informative part of the coefficient profile is the contrast: equal coefficients can completely hide the orientation change, whereas unequal coefficients expose it.

\begin{corollary}[Two-atom coefficient contrast]
\label{cor:two-atom-contrast}
Let $r=2$ and write
\[
    x_1=\bar\beta+\frac d2,
    \qquad
    x_2=\bar\beta-\frac d2 .
\]
Fix $d_0\ne0$ and a sign-preserving compact coefficient set satisfying
$|d-d_0|\le\epsilon_d\le|d_0|/2$. For the orientation orbit that rotates
$v_1-v_2$ while fixing $v_1+v_2$,
\begin{equation}
\begin{aligned}
&\inf_{\bar\beta',d'\in\mathbb R}
\left\|
    \mu_{\phi,\bar\beta,d}
    -
    \mu_{\phi',\bar\beta',d'}
\right\|_2
\\
&\quad=
\frac{
    \lambda_\star s|d|
    \|v_1-v_2\|_2
}{2}
|\sin(\phi-\phi')|.
\end{aligned}
\label{eq:two-atom-profile}
\end{equation}
Hence
\[
    \chi_s
    \asymp
    \frac{
        \lambda_\star
        \|v_1-v_2\|_2
        |d_0|
    }{2},
\]
and calibration and deployment information become comparable at
\begin{equation}
    |d_0|
    \asymp
    \sigma\sqrt{\frac NT}\,s^2.
    \label{eq:amplitude-crossover}
\end{equation}
When $d=0$, the deployment distribution is exactly invariant along this
orientation path, so additional deployment replicates provide no orientation
information and the resolution remains calibration-limited.
\end{corollary}

The corollary therefore gives a concrete interpretation of the general secant condition.
The number of deployment replicates alone does not determine whether deployment improves physical resolution; the active coefficient profile must also reveal the orientation change. In the equal-coefficient case, it does not, whereas a sufficiently large coefficient contrast allows deployment information to compete with the $Ns^6$ calibration information.

\paragraph{Extension beyond Bernoulli--Gaussian coding.}
The $s^3$ orientation sensitivity underlying the calibration bottleneck was derived above under the Bernoulli--Gaussian calibration model.  
To check that this mechanism is not a special consequence of that particular coding law, we also analyze a fixed known exchangeable sparse Gaussian coefficient law.  
The same local geometry yields $s^3$ orientation sensitivity and $Ns^6$ information scaling under a corresponding nondegeneracy condition.
The model, moment formulas, and proofs are given in Supplementary Sections~\ref{sec:supp-known-coefficient-law}, \ref{sec:supp-moment-identification}, and \ref{sec:supp-proof-training-geometry}.

\subsection{Summary of the resolution regimes}
\label{sec:operational-meaning}

Figure~\ref{fig:intro_pipeline:b} summarizes the inferential hierarchy qualitatively: the data must first resolve the local explanation, then the participating atom support, and finally the physical mapping of those atoms.
The results above provide the corresponding quantitative information scales:
\[
\begin{array}{@{}cl@{}}
I_G^{(r)}
&:\quad
\text{Which local explanation}\\
&\quad\text{is supported?}\\
I_S
&:\quad
\text{Which atoms participate?}\\
I_D=Ns^6
&:\quad
\text{How precisely can those atoms}\\
&\quad\text{be mapped physically?}
\end{array}
\]
In the worst case, failure to resolve an earlier distinction prevents the
later, finer conclusion from being certified.
Thus unresolved local explanations can leave order-one physical uncertainty,
unresolved atom support can leave uncertainty at the coherent-block scale
$s$, and after both deployment-side ambiguities are resolved, the remaining
physical resolution is governed by
$s\wedge\frac{1}{\sqrt N\,s^2}$.

The coefficient-profiled task secant does not define an additional stage in Figure~\ref{fig:intro_pipeline:b}. 
It determines whether deployment
replication can help with the final physical-mapping stage.  
When $\chi_{\mathrm{tan}}=0$, orientation changes can be absorbed by coefficient adjustment and the physical mapping remains calibration-limited.  
When $\chi_{\mathrm{tan}}>0$, deployment observations contribute additional orientation information of order
\[
    \frac{Ts^2}{\sigma^2}\chi_{\mathrm{tan}}^2.
\]

The finite-bank procedure developed next follows the same qualitative
hierarchy, but it does not evaluate $I_G^{(r)}$, $I_S$, $I_D$, or
$\chi_{\mathrm{tan}}$ as online query scores.  Instead, AEB determines which
level can be reported from universal conditions over the remaining possible
explanations and, for ambiguity statements, from opposing admissible
witnesses.

%% file: sections/04_theory_to_aeb.tex
\section{Active Endpoint Bracketing for Finite-Bank Physical-Support Inference}
\label{sec:finite-bank-controller}

The continuous correspondence in Sections~\ref{sec:problem}
and~\ref{sec:results} supplies the statistical benchmark. Here candidate
explanations form a finite bank $\mathcal B=\{e_1,\ldots,e_M\}$ fixed before
held-out evaluation, and AEB asks whether unevaluated candidates can still
change the requested physical report. Each $e=(D,S,\eta)$ maps to an
application-specific target $\varphi_{\mathcal B}(e)\in\mathfrak V_{\rm app}$;
the bank may encode several simultaneously active components and need not
discretize the continuous model.

\subsection{Physical assertions on a finite explanation bank}

A truth-level assertion must hold for the represented data-generating
candidate, whereas an ambiguity assertion records coexisting admissible
explanations. Represent a truth-level assertion $g$ by
\begin{equation}
    q_g:\mathcal B\longrightarrow\{0,1\},
    \label{eq:pointwise-truth-predicate}
\end{equation}
where $q_g(e)=1$ means that $e$ satisfies it. For region $R$, active elements
$\Theta_R(e)$, fine cell $F$, and sector $C$, examples are
\begin{equation}
\begin{aligned}
    q_{\mathrm{fine}(R,F)}(e)
    &=
    \mathbf 1\{\Theta_R(e)\ne\varnothing,\ \Theta_R(e)\subseteq F\},\\
    q_{\mathrm{sector}(R,C)}(e)
    &=
    \mathbf 1\{\Theta_R(e)\ne\varnothing,\ \Theta_R(e)\subseteq C\}.
\end{aligned}
\end{equation}
Thus $F\subseteq C$ permits safe coarsening from fine to sector resolution.
Represented-scale absence is another truth-level predicate; its full form is
given in Supplementary Section~\ref{sec:supp-aeb-protocol}.

\subsection{Candidate evaluation and on-bank truth retention}

For held-out sample $W=(W_1,\ldots,W_m)$, the evaluator classifies a queried
candidate as
\[
    \psi_W(e)
    \in
    \{
        \textsc{admissible},
        \textsc{rejected},
        \textsc{indeterminate}
    \},
\]
with indeterminate candidates retained as possible explanations.

Fix a candidate-wise error level $\alpha_{\mathcal B}\in(0,1)$.
For a fully specified on-bank candidate $e$, suppose the held-out observations
are independent with density $f_e$ under $e$, and a proposal density
$f_{\widehat e}\ll f_e$ is constructed using data independent of $W$.
At a predeclared set of checkpoints $\mathcal T$, we define the likelihood-ratio process
\begin{equation}
    E_t(e)
    =
    \prod_{i=1}^{t}
    \frac{f_{\widehat e}(W_i)}{f_e(W_i)},
    \qquad
    t\in\mathcal T,
    \qquad
    E_0(e)=1 .
    \label{eq:conditional-likelihood-ratio-process}
\end{equation}
Conditional on proposal construction, $E_t(e)$ is a nonnegative mean-one
martingale under $e$, which is rejected only after
\begin{equation}
    \max_{t\in\mathcal T}E_t(e)>\frac{1}{\alpha_{\mathcal B}} .
    \label{eq:candidate-rejection-threshold}
\end{equation}

This yields the candidate-wise retention guarantee:

\begin{corollary}[On-bank candidate retention]
\label{cor:eprocess-retention}
Suppose the candidate laws, data split, proposal construction, and evaluation
checkpoints are fixed before held-out evaluation.
If a candidate is rejected only after
Equation~\eqref{eq:candidate-rejection-threshold}, then for every data-generating
candidate $e_\star\in\mathcal B$,
\begin{equation}
    \mathbb P_{e_\star}
    \left\{
        e_\star
        \text{ is ever rejected}
    \right\}
    \le
    \alpha_{\mathcal B}.
    \label{eq:on-bank-truth-retention}
\end{equation}
\end{corollary}

This protects the represented true candidate, not all candidates
simultaneously; implementation details are in Supplementary
Section~\ref{sec:supp-aeb-protocol}.

\subsection{Witnessed and possible explanation sets}

AEB maintains witnessed and possible sets
\begin{equation}
\begin{aligned}
    L_k
    &=\{e\in\mathcal B:e\text{ is queried and admissible}\},\\
    U_k
    &=\{e\in\mathcal B:e\text{ has not been rejected}\}.
\end{aligned}
\label{eq:aeb-lower-upper}
\end{equation}
Thus unqueried and indeterminate candidates remain in $U_k$. If
\[
    \mathcal A_{\mathcal B}(W)
    =\{e\in\mathcal B:\psi_W(e)=\textsc{admissible}\}
\]
is the exhaustive profile, every valid query prefix satisfies
\begin{equation}
    L_k
    \subseteq
    \mathcal A_{\mathcal B}(W)
    \subseteq
    U_k .
    \label{eq:aeb-sandwich}
\end{equation}

A truth-level assertion requires
\begin{equation}
    q_g(e)=1
    \qquad
    \text{for every }e\in U_k,
    \label{eq:universal-truth-predicate}
\end{equation}
whereas ambiguity requires opposing admissible witnesses in $L_k$. With
$P_R(e)=\mathbf1\{\Theta_R(e)\ne\varnothing\}$, support ambiguity requires
\begin{equation}
    \exists\,e^+,e^-\in L_k:
    \qquad
    P_R(e^+)=1,
    \qquad
    P_R(e^-)=0 .
    \label{eq:profile-ambiguity-predicate}
\end{equation}

The central deterministic question is whether partial evaluation can certify a
conclusion without knowing the exhaustive profile. The sandwich relation makes
universal assertions over $U_k$ and witnesses in $L_k$ valid for that profile.

\begin{theorem}[Finite-bank certification]
\label{thm:finite-bank-bridge}
Assume the sandwich relation
in Equation~\eqref{eq:aeb-sandwich}.

\begin{enumerate}[label=(\roman*),leftmargin=*]

\item
If a truth-level assertion $g$ satisfies
Equation~\eqref{eq:universal-truth-predicate}, then $g$ holds for every explanation in
the exhaustive finite-bank profile
$\mathcal A_{\mathcal B}(W)$.

\item
If opposing ambiguity witnesses belong to $L_k$, then both belong to
$\mathcal A_{\mathcal B}(W)$, so the exhaustive finite-bank profile contains
the corresponding ambiguity.

\item
If $F\subseteq C$, certification of a fine-cell assertion for $F$ also
certifies the corresponding sector assertion for $C$.

\end{enumerate}

These statements hold at every reached query prefix, including a
data-dependent stopping prefix.
\end{theorem}

Theorem~\ref{thm:finite-bank-bridge} is deterministic. Combining it with
candidate-wise retention yields a statistical false-report bound: a false
truth-level report can occur only if the true on-bank candidate is first
rejected.

\begin{theorem}[On-bank validity of truth-level finite-bank reports]
\label{thm:on-bank-truth-validity}
Suppose the reported truth-level assertion $\widehat g$ is emitted only when
its predicate satisfies Equation~\eqref{eq:universal-truth-predicate}, and
suppose the data-generating explanation $e_\star$ belongs to the bank.
Under Equation~\eqref{eq:on-bank-truth-retention},
\begin{equation}
    \sup_{e_\star\in\mathcal B}
    \mathbb P_{e_\star}
    \left\{
        \widehat g
        \text{ is reported and }
        q_{\widehat g}(e_\star)=0
    \right\}
    \le
    \alpha_{\mathcal B}.
    \label{eq:on-bank-false-assertion-bound}
\end{equation}
\end{theorem}

Truth-level validity combines universal certification over $U_k$ with
retention of the on-bank truth; ambiguity instead uses witnesses in $L_k$.
Proofs are in Supplementary Section~\ref{sec:supp-aeb-guarantees}.

\subsection{Active endpoint bracketing}

The preceding results say when a partial bank certifies a report; AEB
determines which candidates to evaluate to reach such a certificate.

Fix a requested report family $\mathcal G$ before held-out evaluation, and let
$\mathcal O_{\mathcal G}$ denote its allowed reports.  Define
\[
    \mathsf{Cert}_{\mathcal G}(L,U)
    \in
    \mathcal O_{\mathcal G}\cup\{\bot\}
\]
as a deterministic rule that returns a report only when $U\ne\varnothing$,
truth-level fields are universal
over $U$, and ambiguity fields have their prescribed witnesses in $L$;
otherwise it returns $\bot$.
A predeclared priority rule $\pi_{\mathcal G}(e; L, U)$ orders the unqueried
candidates that can still change an unresolved certificate.  Both rules and
their deterministic tie-breaking are fixed before held-out evaluation.

Initialize $L=\varnothing$ and $U=\mathcal B$. Until certification or the query
limit, evaluate the unqueried candidate of highest predeclared priority, add an
admissible candidate to $L$, remove a rejected candidate from $U$, and retain
an indeterminate candidate in $U$. Recompute the certificate after each query
and return it when available. An empty profile is returned iff
$U=\varnothing$; any other unresolved or budget-limited state returns
\textsc{abstain}. Full pseudocode is given in Supplementary
Algorithm~\ref{alg:supp-aeb}.

Adaptive ordering and stopping preserve Theorem~\ref{thm:finite-bank-bridge}
because every emitted field satisfies the same certificate. Evaluating $Q$
candidates costs $O(QC_{\rm eval})$; no worst-case sublinear-query guarantee is
claimed.

\subsection{Scope of the finite-bank guarantee}
\label{sec:finite-bank-scope}

Deterministic certification is relative to exhaustive evaluation of the same
bank, and statistical validity is on-bank. No claim is made that the bank
approximates or outer-covers the continuous parameter space or gives off-bank
coverage. Application-specific banks, evaluators, certificates, and priorities
are described in Section~\ref{sec:experimental-evaluation} and the supplement.

%% file: sections/05_experiments.tex
\section{Numerical Evaluation and Discussion}
\label{sec:experimental-evaluation}

We present three synthetic studies of the information mechanism, the physical
consequences of dictionary uncertainty, and AEB computation.

\subsection{Theory-guided check of the information mechanism}

\paragraph{Design and rationale.}
This controlled calculation tests the predicted $s^6$ calibration scaling and
coefficient-dependent deployment information in the balanced $q=4$, $r=2$
Bernoulli--Gaussian submodel. Its exact 32-component calibration mixture
isolates orientation without support or optimization effects; stored
orientation perturbations are evaluated by balanced midpoint Monte Carlo.
Equal coefficients provide the deployment-invariance control, whereas unequal
coefficients expose the orientation change predicted by
Theorem~\ref{thm:task-symmetry}. This is a fixed-grid mechanism calculation,
not a collection of independent coverage trials; the complete collapse
diagnostic is Figure~\ref{fig:supp-collapse-diagnostic} in the supplement.

\paragraph{Results and interpretation.}
The Jeffreys-divergence log--log slope is $5.935$, with $R^2=0.99999$ and a
2,000-replicate paired-batch Monte Carlo stability range
$[5.925,5.947]$. The maximum equal-coefficient residual is
$6.8\times10^{-16}$, while the unequal-coefficient analytical identity has
maximum relative error $7.0\times10^{-15}$. Thus the calculation reproduces
the sixth-order calibration mechanism and the coefficient-dependent deployment
distinction. The secondary collapse spread is $0.0167$, narrowly above the
prespecified tolerance $0.015$, so this study is used only as a mechanism
illustration rather than numerical coverage evidence.

\begin{manuscriptwidefigure}
  \centering
  \subfloat[Calibration orientation scaling.  Exact-mixture
  Jeffreys-divergence estimates follow the $s^6$ reference; the fitted slope,
  paired-batch Monte Carlo stability interval, and $R^2$ are computed on the
  stored scale grid.\label{fig:theorem-native-mechanism:a}]{%
    \includegraphics[width=.41\textwidth]
    {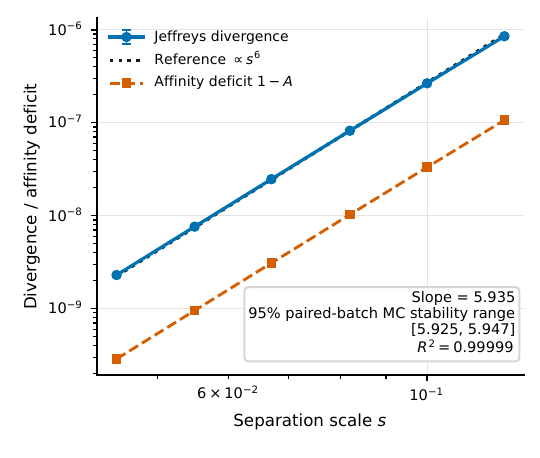}%
  }
  \hfill
  \subfloat[Task-dependent deployment information.  At $s=0.1$, the
  coefficient-profiled residual stays at machine precision for equal
  coefficients and increases with orientation perturbation for unequal
  coefficients. Across the full stored grid, the maximum equal-coefficient
  residual is $6.82\times10^{-16}$ and the maximum analytical-identity
  relative error is $7.01\times10^{-15}$.\label{fig:theorem-native-mechanism:b}]{%
    \includegraphics[width=.41\textwidth]
    {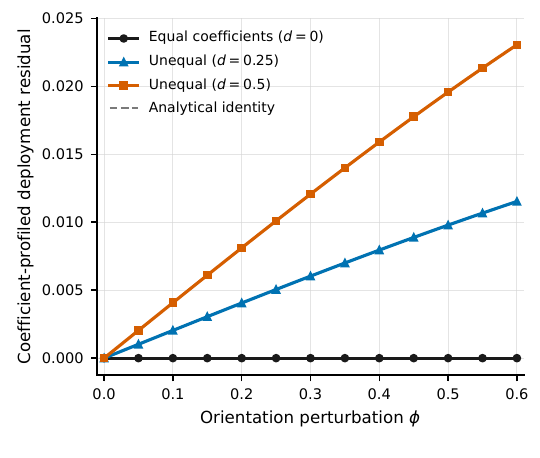}%
  }
  \caption{Theory-guided check of the information mechanism on the stored
  controlled calculation grid.  This is a mechanism illustration, not a
  numerical coverage study for the continuous confidence correspondence.}
  \label{fig:theorem-native-mechanism}
\end{manuscriptwidefigure}

\subsection{Four-region application: physical consequences of dictionary uncertainty}

\paragraph{Design and rationale.}
This study asks whether selecting one dictionary--support explanation can
produce a physical conclusion finer than the complete finite-bank analysis.
Region A is an isolated persistent fine-localization control; B is a coherent
persistent group requiring group-level resolution; C is weak and optional,
creating presence--absence ambiguity; and D is an optional interferer for
represented-scale absence and D-present controls. AEB is compared with
exhaustive evaluation of the same 216-explanation bank and with a point-valued
plug-in selector that chooses one explanation before physical mapping. The
full bank, data split, source scales, comparator, and query policy are in
Supplementary Section~\ref{sec:supp-four-region-protocol}.

\paragraph{Results and interpretation.}
At the primary cap $162/216=0.75$, AEB matches the exhaustive same-bank result
in $10/11$ A-fine profiles, $5/5$ B group/sector profiles, $5/5$ C-ambiguity
profiles, and $10/10$ represented-scale D-absence profiles, with $0/3$ false
D-absence conclusions in the D-present controls. One weak-C dataset has an
empty exhaustive finite-bank profile; it is excluded from truth-relative
utility rates but its query cost is retained. In the five eligible completed
weak-C profiles, the plug-in selector gives unsupported fine B localization in
$5/5$ cases and unsupported definitive C-absence in $5/5$ cases; in the one
remaining A-fine profile, AEB abstains rather than localizing incorrectly.
Figure~\ref{fig:application-result} shows the representative and aggregate
discrepancies. These outcomes show that retaining dictionary uncertainty is
not generic conservatism: it preserves fine localization and absence when
supported, but prevents unsupported refinement when competing explanations
remain.

\begin{manuscriptwidefigure}
  \centering
  \subfloat[Representative weak-C physical report.  AEB and exhaustive
  same-bank evaluation agree on a group-level conclusion in B and support
  ambiguity in C, whereas the point-valued plug-in selector reports fine B
  localization and definite C absence.]{%
    \includegraphics[width=.41\textwidth]
    {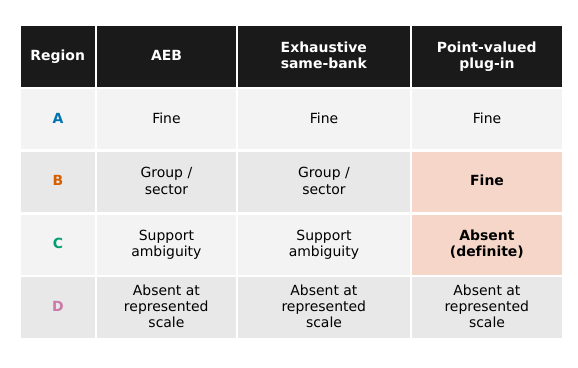}%
  }
  \hfill
  \subfloat[Point-valued plug-in selector discrepancies.  Unsupported finer B
  localization and unsupported definitive C-absence occur in all five eligible
  completed weak-C profiles.]{%
    \includegraphics[width=.41\textwidth]
    {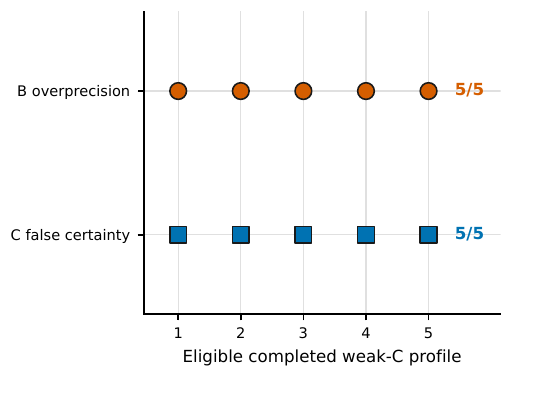}%
  }
  \caption{Four-region physical reports relative to the represented
  216-explanation bank. The representative weak-C report shows that AEB
  preserves the exhaustive same-bank group/ambiguity resolution while the
  point-valued plug-in gives unsupported fine B localization and definite C
  absence; the aggregate panel shows both discrepancies in all five eligible
  completed weak-C profiles. The study uses 15 fresh datasets, one of which
  has an empty exhaustive finite-bank profile.}
  \label{fig:application-result}
\end{manuscriptwidefigure}

\subsection{Global finite-bank study: certification fidelity and query efficiency}

\paragraph{Design and rationale.}
This study asks how much of a finite bank AEB must evaluate to reproduce or
safely coarsen the exhaustive same-bank report. Eighteen independent cases
span low-, intermediate-, and high-information regimes, and three predeclared
reporting profiles per case give 54 traces with different certificate
obligations. Proposal data determine candidate order and policy quantities;
independent held-out data determine candidate classification. Exhaustive
evaluation uses the same bank and classification rule. Outcome counts use all
54 traces, with 34 ambiguity-reference and seven fine-reference traces; full
settings are in Supplementary Section~\ref{sec:supp-global-protocol} and
Table~\ref{supp-tab:global-protocol}.

\paragraph{Results and interpretation.}
At budget $0.50$, AEB returns substantive reports in $41/54$ traces, recovers
$33/34$ ambiguity and $0/7$ fine conclusions, and gives $0/54$ unsafe
finer-than-reference reports. At budget $0.75$, the counts are $54/54$,
$34/34$, $5/7$, and $0/54$; the remaining two fine-reference cases are safely
coarsened. Figure~\ref{fig:global-validation} shows the six stored budgets,
where the four status categories sum to 54 and the unsafe category is always
zero.

\begin{figure}[!t]
  \centering
  \includegraphics[width=0.9\linewidth]
  {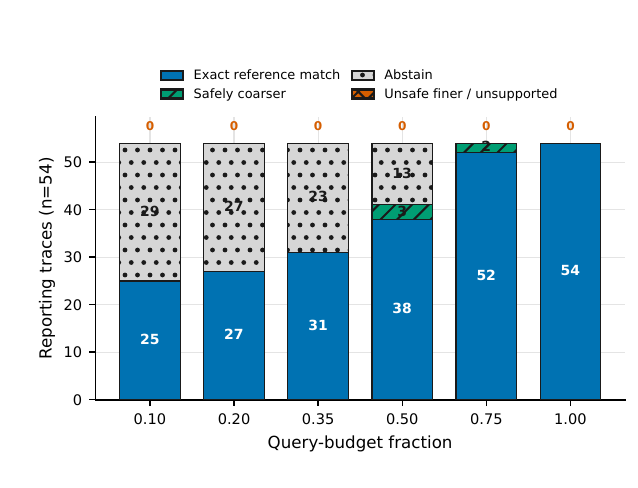}
  \caption{Global AEB status relative to exhaustive same-bank evaluation over
  18 independent cases and three predeclared profiles per case (54 reporting
  traces). At each of the six stored query budgets, exact-reference match,
  safe coarsening, abstention, and unsafe finer-than-reference counts sum to
  54; the unsafe count is zero throughout.}
  \label{fig:global-validation}
\end{figure}

Ambiguity can be certified from opposing witnesses, whereas fine localization
requires eliminating unresolved alternatives, so reduced computation appears
as abstention or safe coarsening rather than unsupported refinement. The median
queried fraction is $0.209$ at both displayed budgets because most traces stop
before either cap; this is decision-dependent query saving, not a worst-case
sublinear-query guarantee.

\subsection{Discussion, scope, and future directions}
\label{sec:discussion}

Together, the studies reproduce the sixth-order calibration mechanism and
coefficient-dependent deployment information, show that dictionary uncertainty
can replace point-valued overprecision by coherent-group or support-ambiguous
conclusions without suppressing supported fine or absence statements, and
show that AEB often certifies the exhaustive same-bank conclusion without
evaluating the full bank while otherwise coarsening or abstaining safely.

Controlled synthetic experiments are appropriate because the claims require a
known local orientation path, exact deployment invariance after coefficient
profiling, and an exhaustive compatible explanation set---quantities that
ordinary real-data benchmarks do not reveal. The continuous theory remains a
supplied local coherent-block benchmark, and the computational results remain
same-bank and on-bank rather than off-bank or worst-case sublinear guarantees.
Natural extensions include continuous-certificate computation, finite-bank
refinement with off-grid guarantees, application-specific bank construction,
broader coherent geometries, and sequential calibration/deployment
acquisition.

%% file: sections/06_conclusion.tex
\section{Conclusion}

Sparse pursuit after dictionary learning can produce a precise coordinate
support even when the corresponding physical interpretation remains
uncertain.  We addressed this gap by formulating physical-support inference
jointly over calibration-compatible dictionaries and deployment-compatible
sparse representations, and by constructing a cross-dictionary confidence
correspondence that propagates dictionary uncertainty into the reported
physical support.

For highly coherent atom configurations with separation scale $s$, we showed
that, once the coherent-block explanation and atom support are resolved, the
minimax physical resolution from $N$ calibration signals satisfies
\[
    \delta_{\mathrm{opt}}(N,s)
    \asymp
    \min\left\{
        s,\frac{1}{\sqrt{N}s^2}
    \right\}.
\]
The corresponding calibration information for physical orientation scales as
$Ns^6$.  Repeated deployment measurements can improve this resolution only
when an orientation change produces a signal component that cannot be absorbed
by adjusting the active coefficients; otherwise the physical resolution
remains calibration-limited.

For computation, we developed active endpoint bracketing (AEB), which applies
the same retain--project--coarsen principle to a specified finite bank of
candidate explanations.  AEB reports a fine physical conclusion only when it
is supported by every remaining possibility, while unresolved alternatives
lead to a coarser conclusion, certified ambiguity, or abstention.  In the
finite-bank studies, point-valued plug-in selection could return a finer
physical interpretation than exhaustive same-bank uncertainty analysis
supported, whereas AEB reproduced or safely coarsened the exhaustive
same-bank conclusion without necessarily evaluating the entire bank. These results provide a resolution-aware view of sparse inference:
the precision of the reported physical support should be determined not only
by the sparse representation selected under one fitted dictionary, but by the
physical distinctions jointly supported by the calibration and deployment
evidence.

%% file: supplement_body.tex
\input{appendices/A_proofs_and_technical_details}
\input{appendices/B_finite_bank_and_numerical_details}

%% file: appendices/A_proofs_and_technical_details.tex

\section{Proofs and Technical Details}
\label{sec:supp-proofs}

This supplement supplies the technical steps behind the main-text results.
Its organization follows the questions answered in Sections~\ref{sec:problem}
and~\ref{sec:results}.  We first make the local coordinate system and
population-moment inversion explicit.  We then show why coherent-block
orientation first appears at cubic order, convert that geometry into
statistical distance and Fisher information, and derive the physical
dictionary-radius bounds used by the confidence correspondence.  The remaining
sections establish finite-sample calibration coverage, measurability,
cross-dictionary separation, the minimax lower bound, and the
deployment-assisted orientation result.

Technical lemmas are placed close to the proof step that needs them so that
each calculation has a visible role in the argument.  A fixed known
exchangeable sparse Gaussian coefficient law is included only as a
supplementary robustness extension; it is not required for the main-text
Bernoulli--Gaussian model or its primary results.

\subsection{Model details deferred from the main text}
\label{sec:supp-model-details}

The main text keeps only the model ingredients needed to understand the
physical target and the retain--project construction.  The proofs require
three additional pieces of bookkeeping: a normalized simplex coordinate
system for comparing nearby dictionaries, explicit second- and fourth-order
population moments for identifying those dictionaries from calibration data,
and a measurable finite-sample calibration region.  We record these ingredients
here before using them in the later upper- and lower-bound arguments.

Section~\ref{sec:supp-local-coordinate-details} fixes the simplex
normalization and quotient metric, Section~\ref{sec:supp-moment-identification}
gives the tensor notation and moment inversion, and
Sections~\ref{sec:supp-mom-region}--\ref{sec:supp-measurability} later provide
one explicit finite-sample implementation.  The known exchangeable-law
extension is separated because its only purpose is to show that the cubic
orientation mechanism is not specific to Bernoulli--Gaussian coding.

\subsubsection{Supplementary fixed known exchangeable sparse Gaussian law}
\label{sec:supp-known-coefficient-law}

Let $J\subseteq[n]$ denote the active set and $K=|J|$.  For fixed known
probabilities $(\pi_0,\ldots,\pi_n)$ and conditional active variances
$(\lambda_0,\ldots,\lambda_n)$, define
\begin{equation}
    \mathbb P(K=k)=\pi_k,
    \qquad
    J\mid K=k\sim\operatorname{Unif}\binom{[n]}k .
    \label{eq:known-law-support}
\end{equation}
Conditional on $(J,K=k)$, the coefficients indexed by $J$ are independent
$\mathcal N(0,\lambda_k)$ variables and all remaining coefficients are zero.
For this supplementary extension, define
\begin{equation}
    a_1
    =
    \frac{\mathbb E(\lambda_KK)}{n},
    \qquad
    \Delta_{\mathcal L}
    =
    \frac{\mathbb E\{\lambda_K^2K(n-K)\}}{n(n-1)}.
    \label{eq:known-law-nondegeneracy}
\end{equation}
We further assume the nondegeneracy bounds
\begin{equation}
    a_1\ge a_{\min}>0,
    \qquad
    \Delta_{\mathcal L}\ge\Delta_{\min}>0.
    \label{eq:law-nondegeneracy}
\end{equation}
These conditions exclude a vanishing second-order signal coefficient and a
coefficient law whose fourth-order contrast disappears.  The corresponding population moment formulas are
stated in Section~\ref{sec:supp-moment-identification} and derived in
Section~\ref{sec:supp-proof-training-geometry}.

\subsubsection{Simplex normalization and quotient dictionary geometry}
\label{sec:supp-local-coordinate-details}

To compare nearby coherent dictionaries without being distracted by arbitrary
atom labels, we use a normalized simplex chart and measure dictionary
separation only after the best coherent-atom permutation.  This is the
coordinate system in which the later cubic invariant and dictionary-radius
bounds are stated.

Use the decomposition
\[
    \mathbb R^q=U\oplus\operatorname{span}\{u\},
    \qquad
    \dim U=q-1,
\]
and choose centered regular-simplex reference vectors
$v_1,\ldots,v_q\in U$ satisfying
\begin{equation}
\begin{gathered}
    \sum_{j=1}^qv_j=0,
    \qquad
    v_i^\top v_j=-\frac{1}{q-1}\quad(i\ne j),\\
    \sum_{j=1}^qv_jv_j^\top=\frac{q}{q-1}I_U.
\end{gathered}
    \label{eq:simplex-reference}
\end{equation}
Together with Equations~\eqref{eq:atom-chart}--\eqref{eq:affine-simplex},
this normalization is only a coordinate choice.  An arbitrary invertible
$L$ produces rotated and anisotropic coherent configurations, and every
affinely independent local $q$-atom configuration admits such an affine
simplex representation.

Let $[D]$ denote the orbit of $D$ under simultaneous permutation of the
coherent atoms.  For two local dictionary orbits, define
\begin{equation}
\begin{aligned}
    d_{\mathcal Q}^{D}([D],[D'])
    &=
    \min_{\pi\in\mathfrak S_q}
    \\
    &\quad
    \max\left\{
        \max_{1\le j\le q}\|z_j-z'_{\pi(j)}\|,
        \|P_a-P_a'\|_F
    \right\}.
\end{aligned}
    \label{eq:physical-dictionary-metric}
\end{equation}
The local sign convention fixes atom signs, so simultaneous permutation is
the remaining dictionary nonidentifiability.  The shell in
Equations~\eqref{eq:shell} and~\eqref{eq:anchor-noise-shell} is compact under
this quotient metric.

For completeness, atom sign is removed from the physical target by identifying
a unit atom $d$ with $[d]=\operatorname{span}\{d\}\in\mathbb{RP}^{q-1}$ and
using
\[
    d_{\mathrm{pr}}([d],[d'])
    =
    \sqrt{1-(d^\top d')^2}.
\]
Let $d_H^{\mathrm{pr}}$ be the induced Hausdorff distance and let
$\mathcal K(E)$ denote the nonempty compact subsets of $E$. The full marked
target space underlying Equation~\eqref{eq:physical-target} is
\begin{equation}
    \mathfrak V_q
    =
    \{0,1\}_{\mathrm{disc}}
    \times
    \mathcal K(\mathbb{RP}^{q-1}),
    \label{eq:target-space}
\end{equation}
with metric
\begin{equation}
    d_{\mathfrak V}\bigl((m,A),(m',A')\bigr)
    =
    |m-m'|+d_H^{\mathrm{pr}}(A,A').
    \label{eq:marked-target-metric}
\end{equation}
The physical-resolution loss used after the local explanation and atom support
are resolved is
\begin{equation}
    d_{\mathrm{phys}}\bigl((m,A),(m',A')\bigr)
    =
    d_H^{\mathrm{pr}}(A,A').
    \label{eq:physical-resolution-loss}
\end{equation}
Its induced set diameter is Equation~\eqref{eq:physical-diameter}.

\subsubsection{Fourth-order tensor notation and population moment inversion}
\label{sec:supp-moment-identification}

The calibration step needs observable quantities that separate aggregate
dictionary shape from the finer orientation information hidden by high
coherence.  The second moment captures the aggregate projector sum, while the
fourth cumulant retains the additional atomwise information needed for
identification.  We therefore make the corresponding tensor notation and
population inversion explicit.

For a unit-norm atom $d_j$, write $P_j=d_jd_j^\top$ and define
\begin{equation}
    \Sigma_2(D)=\sum_{j=1}^nP_j,
    \qquad
    \Sigma_4(D)=\sum_{j=1}^nP_j^{\odot2}.
    \label{eq:projector-powers}
\end{equation}
For a symmetric matrix $C$, its symmetrized tensor square is
\begin{equation}
    (C^{\odot2})_{ijkl}
    =
    \frac13
    \left(
        C_{ij}C_{kl}
        +C_{ik}C_{jl}
        +C_{il}C_{jk}
    \right).
    \label{eq:supp-symmetrized-square}
\end{equation}
With this normalization, a zero-mean Gaussian vector $X$ with covariance $C$
satisfies $\mathbb E(X^{\otimes4})=3C^{\odot2}$.  Hence the fourth cumulant
in Equation~\eqref{eq:training-moments} is explicitly
\[
    K_4
    =
    \mathbb E(Y^{\otimes4})-3M_2^{\odot2}.
\]

Under Bernoulli--Gaussian coding,
\begin{equation}
    M_2=\nu I_q+p\Sigma_2(D),
    \qquad
    K_4=3p(1-p)\Sigma_4(D).
    \label{eq:bg-moments}
\end{equation}
For a fourth-order tensor $T=(T_{ijkl})$, define the double contraction
\begin{equation}
    \operatorname{Tr}_2(T)
    =
    \sum_{i=1}^q\sum_{j=1}^q T_{iijj}.
    \label{eq:supp-double-contraction}
\end{equation}
Since $\operatorname{Tr}_2(P_j^{\odot2})=1$ for a unit-norm atom,
\begin{equation}
    \frac{\operatorname{Tr}_2(K_4)}{3n}
    =
    p(1-p).
    \label{eq:p-trace}
\end{equation}
The restriction $p<1/2$ selects the identifiable branch.

For the supplementary fixed known exchangeable law, define
\begin{equation}
    a_{2d}
    =
    \frac{\mathbb E(\lambda_K^2K)}{n},
    \qquad
    a_{2o}
    =
    \frac{\mathbb E\{\lambda_K^2K(K-1)\}}{n(n-1)}.
    \label{eq:exchangeable-coefficients}
\end{equation}
Then
\begin{equation}
    \Delta_{\mathcal L}=a_{2d}-a_{2o},
    \label{eq:delta-law}
\end{equation}
and
\begin{align}
    M_2
    &=
    \nu I_q+a_1\Sigma_2(D),
    \label{eq:knownlaw-m2}\\
    K_4
    &=
    3\left[
        \Delta_{\mathcal L}\Sigma_4(D)
        +(a_{2o}-a_1^2)\Sigma_2(D)^{\odot2}
    \right].
    \label{eq:knownlaw-k4}
\end{align}
Thus $\Delta_{\mathcal L}>0$ preserves fourth-order information beyond the
aggregate second-order dictionary shape.

Under Bernoulli--Gaussian coding, set
\[
    u=\frac{\operatorname{Tr}_2(K_4)}{3n}.
\]
The population inversion summarized in Lemma~\ref{lem:moment-reduction} is
\begin{equation}
\begin{aligned}
    p&=\frac{1-\sqrt{1-4u}}{2},
    &
    \nu&=\frac{\operatorname{tr}M_2-np}{q},
    \\
    \Sigma_2(D)
    &=
    \frac{M_2-\nu I_q}{p},
    &
    \Sigma_4(D)
    &=
    \frac{K_4}{3p(1-p)}.
\end{aligned}
\label{eq:population-p-nu-inverse}
\end{equation}
For sample moments, one convenient cumulant estimator is
\begin{equation}
    \widehat K_4
    =
    \widehat M_4-3\widehat M_2^{\odot2}.
    \label{eq:sample-cumulant}
\end{equation}
The explicit robust median-of-means construction and its uniform radii are
provided in Section~\ref{sec:supp-mom-region}.  The derivation of all
population formulas above and the proof of Lemma~\ref{lem:moment-reduction}
are included in Section~\ref{sec:supp-proof-training-geometry}.

\subsection{Calibration geometry and orientation information}

We now turn the local dictionary geometry into statistical information.  The
argument has four steps.  First, we identify invariant coordinates whose
distance is equivalent to statistical distance.  Second, we expand the
coherent projectors and show that orientation cancels through quadratic order.
Third, we control the Taylor remainder uniformly.  Finally, we invert the
cubic invariant and profile the regular nuisance directions to obtain the
$s^6$ orientation-information scale.

\subsubsection{Invariant-coordinate distance and all-pair bounds}
\label{sec:supp-invariant-coordinate-bounds}

The main text emphasizes the $s^3$ orientation sensitivity and $s^6$
information scaling.  The proofs also use all-pair bounds in the invariant
coordinates.  Let $f_{p,\theta}$ denote the density of one calibration observation,
$H^2$ the squared Hellinger distance, and $\mathrm{KL}$ the Kullback--Leibler
divergence.  For $\theta=(D,\nu)$, define
\begin{equation}
    F(\theta)=(b,P_a,G_2,G_3,\nu),
    \label{eq:F-coordinate}
\end{equation}
and let $\delta_F$ denote the Euclidean/Frobenius distance in these coordinates.
For the Bernoulli--Gaussian model, we write
\[
    F_p(p,\theta)=(p,F(\theta)),
    \qquad
    \delta_{F,p}^2
    =
    |p-p'|^2+\delta_F^2(\theta,\theta').
\]

Uniformly over $p,p'\in[p_-,p_+]$ and over pairs in the same supplied local
shell,
\begin{equation}
    c\delta_{F,p}^2
    \le
    H^2(f_{p,\theta},f_{p',\theta'})
    \le
    \mathrm{KL}(f_{p,\theta}\|f_{p',\theta'})
    \le
    C\delta_{F,p}^2.
    \label{eq:unknown-p-chord}
\end{equation}
After an optimal permutation $\pi$ of the coherent atoms,
\begin{align}
&\max_j\|z_j-z'_{\pi(j)}\|
+\|P_a-P_a'\|_F
\notag\\
&\qquad\le
C\left(
    \|b-b'\|
    +\|P_a-P_a'\|_F
\right.
\notag\\
&\hspace{28mm}\left.
    +s^{-1}\|G_2-G_2'\|_F
    +s^{-2}\|G_3-G_3'\|_F
\right).
\label{eq:physical-quotient-inverse}
\end{align}

On the centered balanced submodel, define the quotient orientation distance
\[
    d_{O(U)/\mathfrak S_q}(R,R')
    =
    \min_{\pi\in\mathfrak S_q}
    \|R-R'\Pi_\pi\|_F .
\]
Then
\begin{equation}
    H^2(f_{s,R},f_{s,R'})
    \asymp
    \mathrm{KL}(f_{s,R}\|f_{s,R'})
    \asymp
    s^6d_{O(U)/\mathfrak S_q}^2(R,R').
    \label{eq:orientation-chord}
\end{equation}
These bounds provide the technical bridge from moment uncertainty to physical
dictionary uncertainty used in the upper- and lower-bound proofs below.

\subsubsection{Projector chart and exchangeable subset aggregation}
\label{sec:supp-projector-aggregation}

The key question is why orientation is invisible at first and second order.
To answer it, we expand each rank-one projector in the local atom coordinate
and then average over exchangeable active subsets.  The cancellations in this
averaged expansion are what force the first residual orientation term to be
cubic.

For $z\in U$, write
\[
    d(z)=c(z)u+z,
    \qquad
    c(z)=\sqrt{1-\|z\|^2},
\]
and define
\[
    H(z)=P\{d(z)\}-P_u .
\]
Using
\[
    c(z)=1-\frac12\|z\|^2+O(\|z\|^4),
\]
the projector increment admits the total-degree expansion
\begin{equation}
    H(z)=A_1(z)+A_2(z)+A_3(z)+R_4(z),
    \label{eq:sm-tech-s1-1}
\end{equation}
with
\begin{align}
    A_1(z)
    &=
    uz^\top+zu^\top,
    \label{eq:sm-tech-s1-2}\\
    A_2(z)
    &=
    zz^\top-\|z\|^2P_u,
    \label{eq:sm-tech-s1-3}\\
    A_3(z)
    &=
    -\frac12\|z\|^2(uz^\top+zu^\top),
    \label{eq:sm-tech-s1-4}
\end{align}
and, uniformly on the enlarged local cap,
\begin{equation}
    \|R_4(z)\|_F\le C\|z\|^4,
    \qquad
    \|DR_4(z)\|_{\mathrm{op}}\le C\|z\|^3 .
    \label{eq:sm-tech-s1-5}
\end{equation}

Fix a stratum indexed by support size $k$ and separated-atom indicator $A$, let $h=k-A$, and define
\[
    C_0
    =
    \nu I_q+\lambda_k(hP_u+AP_a).
\]
Only strata with $0\le h\le q$ and positive weight are retained.
For a coherent-atom subset $C$, put
\[
    X_m(C)
    =
    \lambda_k\sum_{j\in C}A_m(z_j),
    \qquad m=1,2,3.
\]
Let $\mathcal D_m=D^m\gamma_{C_0}$ denote the $m$th covariance derivative
of the Gaussian density at $C_0$.  Retaining monomials by total degree in the
coherent-atom coordinates gives
\begin{align}
    \mathcal P_{0,C}
    &=
    \gamma_{C_0},\\
    \mathcal P_{1,C}
    &=
    \mathcal D_1[X_1(C)],\\
    \mathcal P_{2,C}
    &=
    \mathcal D_1[X_2(C)]
    +
    \frac12
    \mathcal D_2[X_1(C),X_1(C)],\\
    \mathcal P_{3,C}
    &=
    \mathcal D_1[X_3(C)]
    +
    \mathcal D_2[X_1(C),X_2(C)]
    \notag\\
    &\quad
    +
    \frac16
    \mathcal D_3[X_1(C),X_1(C),X_1(C)].
    \label{eq:sm-tech-s2-4}
\end{align}
Every omitted term has a total coherent-atom degree of at least 4.

The cancellations responsible for the cubic orientation effect appear only
after averaging over exchangeable subsets.  Let $H_1,\ldots, H_q$ belong to a
commutative symmetric tensor algebra, we define
\[
    Q_m=\sum_{j=1}^q H_j^{\odot m},
    \qquad
    X_C=\sum_{j\in C}H_j,
    \qquad
    \pi_m(h)=\frac{(h)_m}{(q)_m},
\]
and let $C$ be uniformly distributed over the $h$-subsets of $[q]$.
Partitioning ordered index tuples by equality pattern yields
\begin{align}
    \mathbb E_hX_C
    &=
    \pi_1Q_1,
    \label{eq:sm-proof-5-4}\\
    \mathbb E_hX_C^{\odot2}
    &=
    (\pi_1-\pi_2)Q_2
    +
    \pi_2Q_1^{\odot2},
    \label{eq:sm-proof-5-5}\\
    \mathbb E_hX_C^{\odot3}
    &=
    (\pi_1-3\pi_2+2\pi_3)Q_3
    \notag\\
    &\quad
    +
    3(\pi_2-\pi_3)Q_1\odot Q_2
    +
    \pi_3Q_1^{\odot3}.
    \label{eq:sm-proof-5-6}
\end{align}

For the projector expansion, every degree-at-most-three subset aggregate
reduces to the coherent-atom power sums
\begin{equation}
\begin{gathered}
    B_1=\sum_jz_j,
    \qquad
    B_2=\sum_jz_j^{\otimes2},\\
    B_3=\sum_jz_j^{\otimes3},
\end{gathered}
    \label{eq:sm-proof-5-9}
\end{equation}
and their products.  The exact affine-simplex identities are
\begin{align}
    B_1
    &=
    qb,
    \label{eq:sm-tech-s3-8}\\
    B_2
    &=
    qb^{\otimes2}
    +
    \frac q{q-1}G_2,
    \label{eq:sm-tech-s3-9}\\
    B_3
    &=
    qb^{\otimes3}
    +
    \frac q{q-1}\operatorname{Sym}_3(b\otimes G_2)
    +
    G_3.
    \label{eq:sm-tech-s3-10}
\end{align}
Hence, no residual orientation coordinate appears at degree one or two;
orientation first enters through the cubic term $G_3=L^{\otimes3}T_q$.

\subsubsection{Gaussian derivative envelope and composed remainder}
\label{sec:supp-gaussian-envelope}

The cubic cancellation is useful only if the omitted fourth- and higher-order
terms remain uniformly smaller.  The next envelope controls derivatives of
the Gaussian mixture components in the weighted $L^2$ norm used for Hellinger
and Fisher calculations, allowing the local polynomial expansion to be
composed with the mixture model without losing uniformity.

The following envelope controls all Taylor terms used in the calibration
density expansion.

\begin{lemma}[Uniform Gaussian derivative envelope]
\label{lem:supp-gaussian-envelope}
Suppose the eigenvalues of positive-definite $A$ lie in a fixed compact
subset of $(0,\infty)$ and
\[
    \|A^{-1/2}(B-A)A^{-1/2}\|_{\mathrm{op}}
    \le
    \rho
    <
    \frac12 .
\]
Then, for $0\le m\le5$,
\begin{equation}
    \int
    \frac{
        |D^m\gamma_A[H_1,\ldots,H_m]|^2
    }{\gamma_B}
    \le
    C_m\prod_{j=1}^m\|H_j\|_F^2 .
    \label{eq:sm-proof-5-16}
\end{equation}
\end{lemma}

\begin{proof}
Each covariance derivative is a polynomial of degree $2m$ in the observation
times $\gamma_A$.  For example,
\[
    D\gamma_A[H](y)
    =
    \frac{\gamma_A(y)}2
    \left\{
        y^\top A^{-1}HA^{-1}y
        -
        \operatorname{tr}(A^{-1}H)
    \right\}.
\]
Moreover,
\[
    2A^{-1}-B^{-1}
    \succeq
    \left(
        2-\frac1{1-\rho}
    \right)A^{-1}
    \succ0,
\]
so the required polynomial moments of the Gaussian quotient are uniformly
finite.
\end{proof}

Let $\mathcal G_{k,A}$ denote the already support-aggregated stratum density,
let $\mathcal P_{\le3,k,A}$ be its total-degree-three polynomial, and write
\[
    R_{\ge4,k,A}
    =
    \mathcal G_{k,A}
    -
    \mathcal P_{\le3,k,A}.
\]
Integral Taylor formulas and Lemma~\ref{lem:supp-gaussian-envelope} give
\begin{align}
    \left\|
        \frac{
            D_zR_{\ge4,k,A}[\dot z]
        }{\sqrt{q_{k,A,\theta'}}}
    \right\|_2
    &\le
    Cs^3\max_j\|\dot z_j\|,
    \label{eq:sm-tech-s4-2}\\
    \left\|
        \frac{
            D_{(P_a,\nu)}R_{\ge4,k,A}
            [\dot P_a,\dot\nu]
        }{\sqrt{q_{k,A,\theta'}}}
    \right\|_2
    &\le
    Cs^4
    \left(
        \|\dot P_a\|_F+|\dot\nu|
    \right).
    \label{eq:sm-tech-s4-3}
\end{align}

Mixture aggregation introduces no inverse minimum-weight factor because, for
nonnegative weights $w_\ell$,
\begin{equation}
    \frac{
        \left(\sum_\ell w_\ell r_\ell\right)^2
    }{
        \sum_\ell w_\ell q_\ell
    }
    \le
    \sum_\ell
    w_\ell\frac{r_\ell^2}{q_\ell}.
    \label{eq:sm-tech-s4-4}
\end{equation}
After optimal dictionary alignment,
\begin{equation}
\begin{aligned}
    s^3\max_j\|\Delta z_j\|
    &\le
    C\left\{
        s^3\|\Delta b\|
        +
        s^2\|\Delta G_2\|
        +
        s\|\Delta G_3\|
    \right\}
    \\
    &\le
    Cs\delta_F.
\end{aligned}
    \label{eq:sm-tech-s4-5}
\end{equation}
Integrating the derivative bounds therefore yields
\begin{equation}
    \left\|
        \frac{
            R_{\ge4}(\theta)-R_{\ge4}(\theta')
        }{\sqrt{f_{\theta'}}}
    \right\|_2
    \le
    Cs\delta_F(\theta,\theta').
    \label{eq:sm-tech-s4-6}
\end{equation}

\subsubsection{Simplex cubic and quotient inverse}
\label{sec:supp-simplex-quotient}

For unit $x\in U$, set $y_j=v_j^\top x$.  Then
\begin{equation}
    \sum_jy_j=0,
    \qquad
    \sum_jy_j^2=\frac q{q-1},
    \qquad
    T_q[x,x,x]=\sum_jy_j^3.
    \label{eq:sm-tech-s5-1}
\end{equation}
The tight-frame map
\[
    x\mapsto(v_1^\top x,\ldots,v_q^\top x)
\]
is a linear isomorphism from $U$ onto the zero-sum hyperplane.
Optimizing $\sum_jy_j^3$ under the constraints in
Equation~\eqref{eq:sm-tech-s5-1} therefore reduces to an exact finite-dimensional
constrained problem.
At a stationary point, the coordinates take at most two values; if $k$
coordinates take the positive value, the objective is
\begin{equation}
    \frac{
        q(q-2k)
    }{
        (q-1)^{3/2}\sqrt{k(q-k)}
    }.
    \label{eq:sm-tech-s5-2}
\end{equation}
This is maximized at $k=1$, corresponding exactly to a simplex vertex.
Hence
\begin{equation}
    \operatorname{Stab}_{O(U)}(T_q)
    =
    \mathfrak S_q.
    \label{eq:sm-tech-s5-3}
\end{equation}

The orbit differential satisfies
\begin{equation}
    \|\Omega\cdot T_q\|_F^2
    =
    3\left(\frac q{q-1}\right)^3
    \|\Omega\|_F^2.
    \label{eq:sm-tech-s5-4}
\end{equation}
Consequently, with
\[
    d_{\mathrm{orb}}([R],[R'])
    =
    \min_{\pi\in\mathfrak S_q}
    \|R-R'\Pi_\pi\|_F,
\]
the constant-rank theorem locally and compactness globally give
\begin{equation}
    d_{\mathrm{orb}}([R],[R'])
    \le
    C
    \|
        R^{\otimes3}T_q
        -
        R'^{\otimes3}T_q
    \|_F .
    \label{eq:sm-tech-s5-5}
\end{equation}

For a general shell point, write the left polar decomposition
\[
    L=P_LO_L,
    \qquad
    P_L=(LL^\top)^{1/2},
\]
and normalize the cubic invariant by
\[
    W(L)=P_L^{-\otimes3}G_3.
\]
Square-root functional calculus gives
\[
    \|P_L-P_L'\|_F
    \le
    Cs^{-1}\|G_2-G_2'\|_F,
\]
while
\begin{equation}
\begin{aligned}
    \|W(L)-W(L')\|_F
    &\le
    C\left[
        s^{-3}\|G_3-G_3'\|_F
        \right.
        \\
    &\hspace{18mm}\left.
        +
        s^{-2}\|G_2-G_2'\|_F
    \right].
\end{aligned}
    \label{eq:sm-tech-s5-6}
\end{equation}
Together with Equation~\eqref{eq:sm-tech-s5-5}, these estimates yield the physical
quotient inverse in Equation~\eqref{eq:physical-quotient-inverse}.

\subsubsection{Collapsed moment operator and all-pair statistical chords}
\label{sec:supp-moment-chord}

We next need a uniform inverse statement: if two calibration laws have nearly
the same observable moments, then their invariant coordinates must also be
close.  The collapsed moment operator isolates the linearized contribution of
the coherent block, separated atom, and noise level, while the remainder is
small for the supplied local scale.  This yields the all-pair statistical
chords used later.

Use the separated-atom chart
\[
    a(w)=\frac{a_0+w}{\|a_0+w\|},
    \qquad
    w\in a_0^\perp,
\]
and write
\[
    \Xi=(B_1,B_2,B_3,w,\nu).
\]
For $X=tu+x$, $x\in U$, define
\begin{align}
\mathsf S_{\dot B}(t,x)
={}&
2t\dot B_1[x]
+\dot B_2[x,x]
-t^2\operatorname{tr}\dot B_2
-t\dot B_3[x,I_U],\\
\mathsf Q_{\dot B}(t,x)
={}&
4t^3\dot B_1[x]
+6t^2\dot B_2[x,x]
-2t^4\operatorname{tr}\dot B_2\\
&+
4t\dot B_3[x,x,x]
-6t^3\dot B_3[x,I_U].
\end{align}
Let
\[
    H(\dot w)
    =
    a_0\dot w^\top+\dot w a_0^\top,
    \qquad
    h_0(t,x)
    =
    \frac12(t+e_1^\top x)^2,
\]
and define
\begin{equation}
\begin{aligned}
    \mathscr A_q(\dot B,\dot w,\dot\nu)
    &=
    \left(
        \dot\nu,\,
        \mathsf S_{\dot B}+H(\dot w)[X,X],\,
        \right.\\
    &\hspace{13mm}\left.
        \mathsf Q_{\dot B}+2h_0H(\dot w)[X,X]
    \right).
\end{aligned}
    \label{eq:sm-tech-s8-5}
\end{equation}

\begin{lemma}[Injectivity of the collapsed moment operator]
\label{lem:supp-collapsed-operator}
For every fixed $q\ge3$, $\mathscr A_q$ is injective.
Consequently there exists $\sigma_q>0$ such that
\begin{equation}
    \|\mathscr A_q h\|
    \ge
    \sigma_q\|h\|.
    \label{eq:sm-tech-s8-10}
\end{equation}
\end{lemma}

\begin{proof}
If $\mathscr A_q(\dot B,\dot w,\dot\nu)=0$, then $\dot\nu=0$ and
\[
    H(\dot w)[X,X]
    =
    -\mathsf S_{\dot B}(t,x).
\]
Substitution into the fourth-order output gives
\[
    \mathsf Q_{\dot B}(t,x)
    =
    (t+e_1^\top x)^2
    \mathsf S_{\dot B}(t,x).
\]
Comparing the coefficients of $t^4,t^0,t^2,t^3,t^1$ successively yields
\[
    \operatorname{tr}\dot B_2=0,\qquad
    \dot B_2=0,\qquad
    2\dot B_1-\dot B_3[\cdot,I_U]=0,
\]
then $\dot B_1=0$, $\dot B_3=0$, and finally $H(\dot w)=0$.
Since the separated-atom tangent map is injective on $a_0^\perp$, $\dot w=0$.
\end{proof}

The nonlinear moment map
\[
    \mathcal R(\theta)
    =
    (\nu,\Sigma_2,\Sigma_4)
\]
satisfies, after quotient alignment,
\begin{equation}
    \mathcal R(\theta)-\mathcal R(\theta')
    =
    \mathscr A_q
    \{
        \Xi(\theta)-\Xi(\theta')
    \}
    +
    \operatorname{Rem}_q(\theta,\theta'),
    \label{eq:sm-tech-s8-12}
\end{equation}
with
\begin{equation}
    \|
        \operatorname{Rem}_q(\theta,\theta')
    \|
    \le
    C_qs
    \|
        \Xi(\theta)-\Xi(\theta')
    \|.
    \label{eq:sm-tech-s8-13}
\end{equation}
Choosing the supplied $s_0$ sufficiently small yields the uniform same-shell
moment chord
\begin{equation}
    c\delta_F(\theta,\theta')
    \le
    \|
        \Delta(\nu,\Sigma_2,\Sigma_4)
    \|
    \le
    C\delta_F(\theta,\theta').
    \label{eq:sm-tech-s8-16}
\end{equation}

For the supplementary fixed known exchangeable law, the same-shell statistical
chord is
\begin{equation}
    c\delta_F^2(\theta,\theta')
    \le
    H^2(f_\theta,f_{\theta'})
    \le
    \mathrm{KL}(f_\theta\|f_{\theta'})
    \le
    C\delta_F^2(\theta,\theta').
    \label{eq:fixed-law-chord}
\end{equation}
To verify this extension, write
\[
    f_\theta
    =
    \sum_{I:w_I>0}
    w_I
    \gamma_{C_I(\theta)}.
\]
The Gaussian envelope and remainder bounds imply, locally in the noise
coordinate and in both ordered directions,
\[
    \chi^2(f_\theta\|f_{\theta'})
    \le
    C\delta_F^2(\theta,\theta').
\]
For a fixed separation in $\nu$, compactness and log-sum absorb the finite
componentwise KL bound into $\delta_F^2$.
For the reverse Hellinger bound, use
\[
    \Psi(Y)
    =
    \{
        \operatorname{vech}(YY^\top),
        \operatorname{symvec}(Y^{\otimes4})
    \}.
\]
Uniform eighth moments imply
\[
    \|
        \mathbb E_\theta\Psi
        -
        \mathbb E_{\theta'}\Psi
    \|
    \le
    CH(f_\theta,f_{\theta'}),
\]
and the explicit moment inverse together with
Equation~\eqref{eq:sm-tech-s8-16} gives
\[
    \delta_F
    \le
    CH(f_\theta,f_{\theta'}).
\]
This proves the all-pair fixed-law chord in
Equation~\eqref{eq:fixed-law-chord}.

For Bernoulli coding,
\[
    w_I(p)
    =
    p^{|I|}(1-p)^{n-|I|}.
\]
The compact branch $p\in[p_-,p_+]$ gives
\[
    \sum_I
    \frac{
        \{w_I(p)-w_I(p')\}^2
    }{w_I(p')}
    \le
    C|p-p'|^2,
\]
while Equations~\eqref{eq:p-trace} and~\eqref{eq:population-p-nu-inverse}
give the reverse Lipschitz control.
Combining the weight and fixed-law changes proves
Equation~\eqref{eq:unknown-p-chord}.

\subsubsection{Point-adapted nuisance geometry and quadratic-mean differentiability}
\label{sec:supp-qmd}

The pairwise chord bounds identify the overall statistical geometry, but the
main theorem also requires the information carried by an infinitesimal
orientation perturbation after regular nuisance parameters are profiled out.
We therefore use a point-adapted tangent chart and verify quadratic-mean
differentiability, which turns the geometric scaling into an efficient Fisher
information bound.

At an interior shell point, we write
\[
    L=sRG,
    \qquad
    R\in O(U),
    \qquad
    G=G^\top>0,
\]
and use the point-adapted path
\begin{equation}
\begin{gathered}
    L_t
    =
    R(sG+tS)e^{t\Omega},
    \\
    \dot L
    =
    R(S+sG\Omega),
    \qquad
    S^\top=S,\quad
    \Omega^\top=-\Omega .
\end{gathered}
    \label{eq:sm-tech-s10-2}
\end{equation}
Together with local Euclidean paths in $(p,\nu,b)$ and the separated-atom
projector tangent, write
\[
    h=(\dot p,\dot\nu,\dot b,\dot P_a,S,\Omega),
\]
with $\dot p$ omitted for the supplementary fixed known law.

Direct differentiation gives
\begin{align}
    DG_2[h]
    &=
    sR(SG+GS)R^\top,
    \label{eq:sm-tech-s10-6}\\
    DG_3[h]
    &=
    \mathcal S_{R,G,s}[S]
    +(sRG)^{\otimes3}(\Omega\cdot T_q),
    \notag\\
    &\quad
    \|\mathcal S_{R,G,s}[S]\|_F
    \le
    Cs^2\|S\|_F .
    \label{eq:sm-tech-s10-8}
\end{align}
The Lyapunov map $S\mapsto SG+GS$ is uniformly coercive on the normalized
compact shape class, and Equation~\eqref{eq:sm-tech-s5-4} controls the orientation
component.  Therefore
\begin{equation}
\begin{aligned}
    \|DF_p(\theta)[h]\|^2
    &\asymp
    |\dot p|^2
    +|\dot\nu|^2
    +\|\dot b\|^2
    \\
    &\quad
    +\|\dot P_a\|_F^2
    +s^2\|S\|_F^2
    +s^6\|\Omega\|_F^2.
\end{aligned}
    \label{eq:sm-tech-s10-5}
\end{equation}

Let
\[
    \ell_{\theta,h}(y)
    =
    \frac{Df_\theta[h](y)}{f_\theta(y)},
    \qquad
    I_\theta(h)
    =
    \mathbb E_\theta
    \ell_{\theta,h}^2.
\]
The Gaussian derivative envelope gives the upper Fisher bound, while the
score--moment identity
\[
    D\mathbb E_\theta\Psi[h]
    =
    \mathbb E_\theta
    \{
        (\Psi-\mathbb E_\theta\Psi)\ell_{\theta,h}
    \}
\]
and the differentiated moment chord give the reverse bound.  Hence
\begin{equation}
    c\|h\|_{\mathrm{mix},\theta}^2
    \le
    I_\theta(h)
    \le
    C\|h\|_{\mathrm{mix},\theta}^2,
    \label{eq:sm-tech-s10-11}
\end{equation}
where the mixed norm is the right-hand side of
Equation~\eqref{eq:sm-tech-s10-5}.

On every compact interior subbranch, the finite Gaussian mixture is
quadratic-mean differentiable in this chart.
Indeed, the component weights and covariance matrices are twice continuously
differentiable, the covariance spectra are uniformly bounded above and away
from zero, and the Gaussian envelope controls first and second chart
derivatives in $L^2(f_\theta^{-1})$.
Taylor's theorem therefore gives
\[
    \left\|
        \sqrt{f_{\theta+u}}
        -
        \sqrt{f_\theta}
        -
        \frac12
        \frac{Df_\theta[u]}{\sqrt{f_\theta}}
    \right\|_2
    =
    o(\|u\|).
\]

Define the nuisance-profiled orientation information by
\begin{equation}
    I_{R\mid\eta,\theta}(\Omega)
    =
    \inf_{(\dot p,\dot\nu,\dot b,\dot P_a,S)}
    I_\theta
    (
        \dot p,\dot\nu,\dot b,\dot P_a,S,\Omega
    ).
    \label{eq:sm-tech-s10-16}
\end{equation}
Taking the infimum in the lower half of
Equation~\eqref{eq:sm-tech-s10-11}, and setting nuisance tangents to zero for the
upper half, gives
\begin{equation}
    cs^6\|\Omega\|_F^2
    \le
    I_{R\mid\eta,\theta}(\Omega)
    \le
    Cs^6\|\Omega\|_F^2,
    \label{eq:sm-tech-s10-17}
\end{equation}
which establishes Equation~\eqref{eq:profiled-orientation-information}.

\subsubsection{Proof of the calibration-geometry results}
\label{sec:supp-proof-training-geometry}

The preceding ingredients now assemble directly: the moment formulas give
identification, the subset expansion identifies the first orientation term,
the Gaussian envelope controls the remainder, and the tangent geometry gives
the profiled Fisher information.

\begin{proof}[Proof of Lemma~\ref{lem:moment-reduction} and
Theorem~\ref{thm:training-geometry}]

Conditional on the active set, one calibration observation is Gaussian with
covariance
\[
    C_I=\nu I_q+\lambda_K\sum_{j=1}^n I_jP_j.
\]
Exchangeability gives, for $i\ne j$,
\begin{equation}
    \mathbb E(\lambda_KI_j)=a_1,\qquad
    \mathbb E(\lambda_K^2I_j)=a_{2d},\qquad
    \mathbb E(\lambda_K^2I_iI_j)=a_{2o}.
    \label{eq:supp-exchangeable-coefficients}
\end{equation}
Therefore
\[
    M_2=\nu I_q+a_1\Sigma_2(D)
\]
and, by the conditional Gaussian fourth-moment identity,
\[
    K_4
    =
    3\left\{
        \mathbb E(C_I^{\odot2})
        -
        (\mathbb EC_I)^{\odot2}
    \right\}.
\]
Separating diagonal and off-diagonal atom pairs yields
\[
    \mathbb E
    \left(
        \lambda_K\sum_jI_jP_j
    \right)^{\odot2}
    =
    a_{2o}\Sigma_2^{\odot2}
    +(a_{2d}-a_{2o})\Sigma_4,
\]
so Equations~\eqref{eq:knownlaw-m2}--\eqref{eq:knownlaw-k4} hold and
\begin{equation}
    \Delta_{\mathcal L}
    =
    a_{2d}-a_{2o}
    =
    \frac{
        \mathbb E\{\lambda_K^2K(n-K)\}
    }{n(n-1)}.
    \label{eq:supp-delta-law}
\end{equation}
Under Equation~\eqref{eq:law-nondegeneracy}, the resulting population-moment map is
uniformly invertible on the compact known-law class.  For
Bernoulli--Gaussian coding,
$a_1=a_{2d}=p$ and $a_{2o}=p^2$, and
Equations~\eqref{eq:p-trace}--\eqref{eq:population-p-nu-inverse} give the stated
identification of $p$, $\nu$, $\Sigma_2$, and $\Sigma_4$.
This proves Lemma~\ref{lem:moment-reduction}.

The remaining claims follow from the technical results immediately above.
The projector expansion and exchangeable subset aggregation show that all
orientation dependence through degree two cancels after support averaging,
while the first residual orientation term is the cubic invariant
$G_3=L^{\otimes3}T_q$.  Under Bernoulli--Gaussian coding, the leading
orientation derivative on the centered balanced submodel is proportional to
$p(1-p)s^3(\Omega\cdot T_q)$.  Because
$p\in[p_-,p_+]\Subset(0,1/2)$, this coefficient is uniformly bounded away
from zero, and the simplex-orbit differential
Equation~\eqref{eq:sm-tech-s5-4} makes the derivative nonzero for every
nonzero quotient orientation tangent.  This gives
Equation~\eqref{eq:cubic-density-score}.  For the supplementary fixed known
law, the same calculation replaces $p(1-p)$ by
$\Delta_{\mathcal L}$, which is nondegenerate under
Equation~\eqref{eq:law-nondegeneracy}.

The Gaussian derivative envelope, the composed-remainder bound, and the
collapsed moment operator yield the Bernoulli--Gaussian chord in
Equation~\eqref{eq:unknown-p-chord} and, separately, the supplementary
fixed-law chord in Equation~\eqref{eq:fixed-law-chord}.  The simplex cubic
stabilizer and polar decomposition yield the physical quotient inverse
in Equation~\eqref{eq:physical-quotient-inverse}.  Finally, the point-adapted tangent
geometry and quadratic-mean differentiability give
\[
\begin{aligned}
    \|DF_p(\theta)[h]\|^2
    &\asymp
    |\dot p|^2+|\dot\nu|^2+\|\dot b\|^2
    \\
    &\quad+\|\dot P_a\|_F^2
    +s^2\|S\|_F^2+s^6\|\Omega\|_F^2,
\end{aligned}
\]
with the $p$ term omitted for the supplementary fixed known law.  Profiling over nuisance
tangents therefore gives
Equation~\eqref{eq:profiled-orientation-information}, and the centered balanced
submodel gives Equation~\eqref{eq:orientation-chord}.
\end{proof}

\subsection{Confidence correspondence and physical-resolution upper bound}

The main-text upper bound requires three distinct ingredients.  We first give
one explicit calibration region with uniform finite-sample coverage.  We then
verify that the resulting retain--project correspondence is measurable.
Finally, we quantify how deployment explanations separate when the retained
dictionary is known only up to a radius $\rho$.  Together these steps prove
truth retention and the physical-diameter bound.

\subsubsection{Explicit median-of-means calibration region}
\label{sec:supp-mom-region}

The main text deliberately treats calibration coverage as a modular
interface.  To show that this interface is nonempty and achieves the required
$N^{-1/2}$ scale in fixed dimension, we give one explicit robust
median-of-means construction satisfying
Equation~\eqref{eq:training-rectangle}.
Use fixed orthonormal bases of
$\operatorname{Sym}^2(\mathbb R^q)$ and
$\operatorname{Sym}^4(\mathbb R^q)$, of dimensions
\begin{equation}
    d_2=\binom{q+1}{2},
    \qquad
    d_4=\binom{q+3}{4}.
    \label{eq:sm-tech-s13-1}
\end{equation}
Let
\[
    X_2=\operatorname{symvec}_2(Y^{\otimes2}),
    \qquad
    X_4=\operatorname{symvec}_4(Y^{\otimes4}).
\]
Define
\[
    \Lambda
    =
    \nu_+
    +
    n\lambda_{\max},
\]
where $\lambda_{\max}=1$ for the Bernoulli--Gaussian model used in the main
text; for the supplementary fixed known-law extension, take
$\lambda_{\max}=\max_k\lambda_k$.
Every conditional covariance is bounded by $\Lambda I_q$, hence
\begin{align}
    v_2^2
    &:=
    \sup_\theta
    \mathbb E_\theta\|Y\|^4
    \le
    \Lambda^2q(q+2),
    \label{eq:sm-tech-s13-4}\\
    v_4^2
    &:=
    \sup_\theta
    \mathbb E_\theta\|Y\|^8
    \le
    \Lambda^4q(q+2)(q+4)(q+6).
    \label{eq:sm-tech-s13-5}
\end{align}

Choose $\alpha_2,\alpha_4>0$ with
$\alpha_2+\alpha_4\le\alpha_D$ and set, for $k\in\{2,4\}$,
\begin{equation}
    B_k
    =
    \left\lceil
        8\log\frac{2d_k}{\alpha_k}
    \right\rceil,
    \qquad
    m_k
    =
    \left\lfloor
        \frac N{B_k}
    \right\rfloor.
    \label{eq:sm-tech-s13-6}
\end{equation}
When $m_k\ge1$, split the first $m_kB_k$ observations into deterministic
blocks, compute every coordinate block mean, and take its lower median under a
fixed tie rule.
The scalar median-of-means inequality plus a coordinate union bound gives
\begin{align}
    \mathbb P_\theta^N
    \left\{
        \|\widehat M_2-M_2\|_F
        \le
        \epsilon_{2,N}
    \right\}
    &\ge
    1-\alpha_2,
    &
    \epsilon_{2,N}
    &=
    2v_2
    \sqrt{
        \frac{d_2}{m_2}
    },
    \label{eq:sm-tech-s13-7}\\
    \mathbb P_\theta^N
    \left\{
        \|\widehat M_4-M_4\|_F
        \le
        \epsilon_{4,N}^{\rm raw}
    \right\}
    &\ge
    1-\alpha_4,
    &
    \epsilon_{4,N}^{\rm raw}
    &=
    2v_4
    \sqrt{
        \frac{d_4}{m_4}
    }.
    \label{eq:sm-tech-s13-8}
\end{align}

Because normalized symmetrization is an orthogonal projection,
\[
    \|A\odot B\|_F
    \le
    \|A\|_F\|B\|_F.
\]
Also
\[
    \sup_\theta\|M_2(\theta)\|_F
    \le
    B_2^\star
    :=
    \sqrt q\,\Lambda.
\]
Therefore
\begin{equation}
    \|\widehat K_4-K_4\|_F
    \le
    \epsilon_{K,N}
    :=
    \epsilon_{4,N}^{\rm raw}
    +
    3\epsilon_{2,N}
    \left(
        2B_2^\star+\epsilon_{2,N}
    \right).
    \label{eq:sm-tech-s13-12}
\end{equation}
Thus Equation~\eqref{eq:training-rectangle} holds with probability at least
$1-\alpha_D$, uniformly over the fixed-dimensional class.
For fixed confidence level and fixed $q$, both certified radii are
$O(N^{-1/2})$ once
\[
    N
    \ge
    N_{\rm MoM}
    :=
    \max(B_2,B_4).
\]
Below this threshold, the full compact calibration class is retained.

This construction is only one valid realization of the modular calibration
coverage interface; any sharper estimator may replace it after its own
uniform coverage radii are established.

\subsubsection{Measurability of the confidence correspondence}
\label{sec:supp-measurability}

Coverage statements for a set-valued procedure require the reported
correspondence and its target-membership events to be measurable.  Because the
parameter space is quotiented by a finite permutation group and the deployment
fit involves compact minimization, this can be verified using standard
measurable-selection machinery.

Let $\widetilde{\mathcal D}_{q,s}$ be the closed, canonically signed labelled
dictionary shell, and define the compact quotient
\begin{equation}
    \mathcal M_{q,r,s}^p
    =
    \frac{
        [p_-,p_+]\times[\nu_-,\nu_+]
        \times\widetilde{\mathcal D}_{q,s}
        \times
        \left(
            \binom{[q]}r\sqcup\{\partial\}
        \right)
    }{\mathfrak S_q}.
    \label{eq:sm-tech-s11-5}
\end{equation}
The finite group acts simultaneously on coherent-atom labels and supports.
No freeness assumption is needed.

For $m=[p,\nu,D,S]$, define
\begin{equation}
L(z,m)=
\begin{cases}
\displaystyle
\min_{x_j\in[\beta_-,\beta_+],\,j\in S}
\left\|
    z-\sum_{j\in S}x_jd_j
\right\|_2,
& S\ne\partial,\\[3mm]
\displaystyle
\min_{\gamma\in\Gamma_A}
\|z-\gamma a\|_2,
& S=\partial.
\end{cases}
\label{eq:sm-tech-s11-7}
\end{equation}
Compact minimization shows that $L$ is jointly continuous on
$\mathbb R^q\times\mathcal M_{q,r,s}^p$.
The physical target map $\vartheta$ is likewise continuous.

For data $\omega=(y,z^T)$, define
\begin{equation}
\begin{aligned}
    \mathcal R(\omega)
    &=
    \left\{
        m\in\mathcal M_{q,r,s}^p:
        \right.\\
    &\hspace{8mm}\left.
        \pi_{\rm tr}(m)
        \in
        \widehat{\mathcal K}_{q,s}^p(y),
        \quad
        L(\bar z,m)\le\tau_{T,q}
    \right\}.
\end{aligned}
    \label{eq:sm-tech-s11-10}
\end{equation}
Its graph is Borel and its sections are compact.
The Arsenin--Kunugui projection theorem therefore makes all open-set hit
events Borel, including the nonempty-profile event.
Consequently
\[
    \widehat{\mathfrak C}
    =
    \vartheta\{\mathcal R(\omega)\}
\]
on the nonempty event, patched by the fixed admissible singleton on its
complement, defines a Borel measurable compact-valued correspondence.
The empty-profile indicator is also Borel measurable.

The target-membership relation is closed in
$\mathcal K(\mathfrak V_q)\times\mathfrak V_q$, so the conditional
deployment-coverage kernel is Borel in the calibration sample.
This justifies the conditional and marginal coverage statements in
Theorem~\ref{thm:honest-region}.

\subsubsection{Cross-dictionary separation constants}
\label{sec:supp-separation-constants}

The retained dictionary is not known exactly, so deployment alternatives must
be separated uniformly over nearby dictionaries rather than only at a fixed
dictionary.  The following calculations produce the two margins used in the
main text: $m_S(\rho)$ for competing coherent-block supports and
$m_G(\rho)$ for the coherent-block versus separated alternative.

Let
\[
    c_u=\sqrt{1-C_0^2s_0^2},
    \qquad
    \lambda_V
    =
    \kappa_-\sqrt{\frac q{q-1}}.
\]
For $h\in\mathbb R^q$, write
\[
    A=\mathbf1^\top h,
    \qquad
    h_0=h-\frac Aq\mathbf1 .
\]
Because the simplex frame is tight,
\[
    \|Vh_0\|
    =
    \sqrt{\frac q{q-1}}\|h_0\|.
\]
Using $\|b\|\le C_0s$ and
$\sigma_{\min}(L)\ge\kappa_-s$ gives a fixed constant $c_D>0$ such that
\begin{equation}
    \left\|
        \sum_{j=1}^qh_jd_j
    \right\|_2
    \ge
    c_Ds\|h\|_2.
    \label{eq:sm-tech-s12-3}
\end{equation}
Hence two distinct size-$r$ supports satisfy
\begin{equation}
    \inf_{x,x'}
    \|
        \mu(D,S,x)
        -
        \mu(D,S',x')
    \|_2
    \ge
    \sqrt2\,c_D\beta_-s .
    \label{eq:sm-tech-s12-10}
\end{equation}
Transport across dictionary distance $\rho$ yields
\begin{equation}
    m_S(\rho)
    =
    \left[
        \sqrt2c_D\beta_-s
        -
        C_Sr\beta_+\rho
    \right]_+,
    \label{eq:sm-tech-s12-13}
\end{equation}
which establishes Equation~\eqref{eq:support-margin}.

For coherent-group versus separated-atom separation, use
\[
    h_a=\frac{u-e_1}{\sqrt2}.
\]
The separated-atom neighborhood implies
\[
    |h_a^\top a|
    \le
    C_as_0,
\]
while every coherent-group mean satisfies
\[
    h_a^\top\mu(D,S,x)
    \ge
    \frac r{\sqrt2}
    \{
        \beta_-c_u
        -
        \beta_+C_0s_0
    \}.
\]
After shrinking $s_0$ if needed,
\begin{equation}
    g_G^0
    :=
    \frac r{\sqrt2}
    \{
        \beta_-c_u
        -
        \beta_+C_0s_0
    \}
    -
    \gamma_+C_as_0
    >
    0.
    \label{eq:sm-tech-s12-16}
\end{equation}
Cross-dictionary transport then gives
\begin{equation}
    m_G(\rho)
    =
    \left[
        g_G^0
        -
        C_G(r\beta_++\gamma_+)\rho
    \right]_+,
    \label{eq:sm-tech-s12-17}
\end{equation}
which establishes Equation~\eqref{eq:parent-margin}.

Finally, any two retained deployment candidates have fitted means within
$2\tau_{T,q}$.
Applying the two margins above to all candidate-pair types yields
\begin{equation}
\begin{aligned}
    \operatorname{diam}_{\mathrm{pr}}
    (
        \widehat{\mathfrak C}
    )
    &\le
    C\rho
    +
    Cs\mathbf1
    \{
        2\tau_{T,q}\ge m_S(\rho)
    \}
    \\
    &\quad
    +
    C\mathbf1
    \{
        2\tau_{T,q}\ge m_G(\rho)
    \},
\end{aligned}
    \label{eq:sm-tech-s12-18}
\end{equation}
which gives the deterministic diameter bound used in
Theorem~\ref{thm:honest-region}.

\subsubsection{Proof of truth retention, coverage, and the physical-resolution upper bound}
\label{sec:supp-proof-honest-region}

We now combine the three ingredients above.  Calibration coverage retains the
true dictionary with probability at least $1-\alpha_D$, deployment coverage
retains the true representation conditionally with probability at least
$1-\alpha_T$, and the separation margins determine how much physical
coexistence can remain among the retained explanations.

\begin{proof}[Proof of Proposition~\ref{prop:profile-well-defined} and
Theorem~\ref{thm:honest-region}]

The explicit median-of-means construction above establishes
Equation~\eqref{eq:training-rectangle}; below its fixed block threshold, the full compact
calibration class is retained.  The measurability results above show that the
joint feasible set has a Borel graph with compact sections and that
$\widehat{\mathfrak C}_{q,r,s}^{p}$ can be chosen as a measurable
compact-valued correspondence.  This proves the structural part of
Proposition~\ref{prop:profile-well-defined}.

Let
\[
    E_D
    =
    \left\{
        (p_\star,\nu_\star,[D_\star])
        \in
        \widehat{\mathcal K}_{q,s}^{p}
    \right\}.
\]
By Equation~\eqref{eq:training-rectangle},
$\mathbb P(E_D)\ge1-\alpha_D$.
Conditional on any calibration realization in $E_D$,
\[
    \|\bar Z-\mu_\star\|_2\le\tau_{T,q}
\]
has probability at least $1-\alpha_T$, uniformly over the allowed deployment
noise levels, and retains the true support and coefficient witness.
Hence
\[
    \vartheta_\star
    \in
    \widehat{\mathfrak C}_{q,r,s}^{p},
    \qquad
    F_{\mathrm{empty}}=0
\]
on the intersection of the two truth-retention events.
This proves Proposition~\ref{prop:profile-well-defined},
Equation~\eqref{eq:two-level-coverage}, and, by independence,
Equation~\eqref{eq:marginal-coverage}.

For two members of the same nonempty calibration region,
Equations~\eqref{eq:bg-moments}--\eqref{eq:population-p-nu-inverse} imply
\begin{equation}
    |p-p'|
    +
    \delta_F\{(D,\nu),(D',\nu')\}
    \le
    C\epsilon_N .
    \label{eq:supp-profile-F-radius}
\end{equation}
Combining this with the physical quotient inverse gives
Equation~\eqref{eq:dictionary-radius}; the certified
$\epsilon_N=O(N^{-1/2})$ rate gives
Equation~\eqref{eq:dictionary-radius-rate}.

It remains to be controlled for the coexistence of retained deployment explanations.
The cross-dictionary separation lemmas above provide the support margin
$m_S(\rho)$ and coherent-group/separated margin $m_G(\rho)$.
Any two retained candidates have fitted means within
$2\tau_{T,q}$ of each other.  Therefore, candidates with different deployment
alternatives cannot coexist when
$2\tau_{T,q}<m_G(\rho)$, and distinct coherent-group supports cannot coexist
when $2\tau_{T,q}<m_S(\rho)$.  When the support margin is positive, the
quotient geometry also gives a unique matching of nearby coherent atoms;
candidates on the same support branch then differ by $O(\rho)$ in physical distance.
Applying these cases with $\rho=\rho_N(s)$ yields
Equation~\eqref{eq:three-scale-upper-bound} and completes the proof.
\end{proof}

\subsection{Minimax physical resolution}

The upper bound shows what the proposed correspondence achieves.  To prove
that the rate is intrinsic, we compare any uniformly valid correspondence on
carefully chosen pairs of nearby statistical models whose physical targets
remain separated.  The generic two-point lemma below converts such
indistinguishability into a lower bound on expected physical diameter.

\subsubsection{Interior amplitude condition and generic two-point bound}
\label{sec:supp-minimax-setup}

The first lower-bound pair compares the coherent-block explanation with the
separated alternative.  To make that comparison vary only the physical
explanation rather than hit an amplitude boundary, we require the projected
coherent-block mean to lie strictly inside the allowed separated-amplitude
range:
\begin{equation}
\left\{
\begin{aligned}
    a_0^\top\sum_{j\in S}x_jd_j:
    S\in\binom{[q]}r,\;
    x_j\in[\beta_-,\beta_+],\;
    \\
    b=0,\;
    L=\lambda_\star sR,\;
    R\in O(U)
\end{aligned}
\right\}
\subset
\operatorname{int}(\Gamma_A).
\label{eq:anchor-amplitude-transversality}
\end{equation}
This condition permits the two deployment alternatives to be paired without
forcing either one to the boundary of its allowed amplitude range.

\begin{lemma}[Coverage at two alternatives implies nonzero physical diameter]
\label{lem:honesty-diameter}
Let $\vartheta(\xi_i)=(m_i,A_i)$, $i=0,1$, and set
\[
    d=d_H^{\mathrm{pr}}(A_0,A_1).
\]
If a random compact correspondence $\widehat C$ has coverage at least
$1-\alpha$ at both parameter values, then
\begin{equation}
    \mathbb E_{\xi_0}
    \operatorname{diam}_{\mathrm{pr}}(\widehat C)
    \ge
    d\left[
        1-2\alpha
        -
        \operatorname{TV}
        \left(
            \mathbb P_{\xi_0}^{N,T},
            \mathbb P_{\xi_1}^{N,T}
        \right)
    \right]_+ .
    \label{eq:supp-two-point-diameter}
\end{equation}
\end{lemma}

\subsubsection{Proof of the minimax lower bound and resolved optimality}
\label{sec:supp-proof-minimax}

The three lower-bound pairs correspond exactly to the three uncertainty
sources in the main text: local explanation, atom support, and physical
orientation.  After proving those lower terms, we return to the upper bound
and show that once the two deployment-side distinctions are resolved, only the
calibration-limited orientation rate remains.

\begin{proof}[Proof of Theorem~\ref{thm:fixed-shell-minimax} and
Corollaries~\ref{cor:known-learned-gap}--\ref{cor:oracle-persistence}]

The generic two-point diameter bound is given in
Lemma~\ref{lem:honesty-diameter}.  We apply it to three two-point submodels.

\paragraph{Coherent-group versus separated alternative.}
On the balanced submodel, choose
$x=\beta_-\mathbf 1_S$, $\sigma=\sigma_+$, and
$\gamma_\star=a_0^\top\mu_F\in\operatorname{int}(\Gamma_A)$.
The calibration laws are identical and
\begin{equation}
    \mathrm{KL}_{N,T}
    =
    \frac{T}{2\sigma_+^2}
    \|
        (I-P_{a_0})\mu_F
    \|_2^2
    \le
    C I_G^{(r)}.
    \label{eq:supp-group-pair-KL}
\end{equation}
The physical-support components remain separated by a fixed positive amount.
Equation~\eqref{eq:supp-two-point-diameter} therefore yields the order-one
lower term.

\paragraph{Wrong-support pair.}
Keep the dictionary fixed, choose two supports sharing $r-1$ atoms, and use
the common coefficient $\beta_-$.  If $i,j$ are the exchanged atoms,
\begin{equation}
    \|\mu_S-\mu_{S'}\|_2^2
    =
    2\beta_-^2\lambda_\star^2s^2
    \frac q{q-1},
    \label{eq:supp-support-pair-mean}
\end{equation}
so the joint KL is $O(I_S)$, while the projective Hausdorff separation of the
physical supports is $O(s)$ and bounded below by a fixed multiple of $s$.
This gives the order-$s$ lower term.

\paragraph{Compensated orientation pair.}
At the centered balanced submodel with equal active coefficients, let
\[
    w_S=\sum_{j\in S}v_j.
\]
Choose a nonzero skew generator $\Omega$ that fixes $w_S$ while moving at
least one active vertex, and normalize it so that
$\max_{j\in S}\|\Omega v_j\|=1$.
For sufficiently small fixed $h_0$, the same-label matching remains optimal
along $R\exp(h\Omega)$ and
\begin{equation}
    cs|h|
    \le
    d_{\mathrm{phys}}
    \{
        \vartheta(D_0,S),
        \vartheta(D_h,S)
    \}
    \le
    Cs|h|,
    \qquad
    |h|\le h_0 .
    \label{eq:supp-orientation-target-secant}
\end{equation}
The complete deployment law is identical along this compensated path, while
the calibration KL is at most $CNs^6h^2$ by
Equation~\eqref{eq:unknown-p-chord} with $p$ held fixed.
Taking
\[
    |h|
    \asymp
    \min\left\{
        h_0,
        (\sqrt N\,s^3)^{-1}
    \right\}
\]
and applying Equation~\eqref{eq:supp-two-point-diameter} yields
\[
    s\wedge\frac{1}{\sqrt N\,s^2}.
\]
Together, the three pairs prove Equation~\eqref{eq:three-gate-lower-bound}.

For the matching upper bound, set
\[
    R_D(s,N)
    =
    s\wedge(\sqrt N\,s^2)^{-1}.
\]
Equation~\eqref{eq:dictionary-radius-rate} gives
$\rho_N\le C_\rho R_D$.
Choose the fixed shell radius and the resolution thresholds so that the
coherent-group/separated indicator in
Equation~\eqref{eq:three-scale-upper-bound} vanishes whenever $I_G^{(r)}$ exceeds its
upper threshold.
If $R_D$ is a fixed fraction of $s$, the support term is already
$O(R_D)$; if $R_D$ is smaller, then $\rho_N=o(s)$, the support margin is
positive, and the upper threshold on $I_S$ removes the support indicator.
Thus the right-hand side of Equation~\eqref{eq:three-scale-upper-bound} is
$O(R_D)$, proving Equation~\eqref{eq:resolved-minimax-rate}.

For Corollary~\ref{cor:known-learned-gap}, set the dictionary uncertainty to
zero.  Then $I_S\to\infty$ separates the fine supports and
$I_G^{(r)}\to\infty$ excludes the separated alternative, so the physical
diameter tends to zero.  In the learned-dictionary problem, the compensated
orientation pair retains an order-$s$ physical displacement whenever
$Ns^6$ remains bounded.

For Corollary~\ref{cor:oracle-persistence}, the same compensated path already
fixes the coherent group, support, equal coefficients, separated atom,
coefficient law, noise levels, and all nonorientation nuisance parameters.
Hence the lower bound survives after these quantities are revealed by an
oracle.
\end{proof}

\subsection{Deployment-assisted orientation information}

The main-text secant identifies when deployment replication contains
orientation information.  Here we make the restricted-orbit rate precise and
then work out the two-atom case, where the abstract secant reduces to a simple
coefficient contrast.

\subsubsection{Restricted orientation orbit for the matching rate}
\label{sec:supp-task-orbit}

A matching rate cannot hold on an arbitrary orientation family without
controlling both physical target displacement and deployment mean
displacement.  We therefore restrict to a compact one-dimensional orientation
orbit on which both secants are uniformly comparable to $|\phi-\phi'|$.

For the matching finite-rate statement in
Theorem~\ref{thm:task-symmetry}, fix interior $p_0,\nu_0$, the separated atom,
a support $S_0$, all nonorientation dictionary coordinates, and known
deployment noise $\sigma$.  Consider the compact orientation orbit
\begin{equation}
\begin{gathered}
    D_\phi
    =
    D\{
        b=0,\,
        L=\lambda_\star sR\exp(\phi\Omega),\,
        P_a=P_{a_0}
    \},
    \\
    |\phi|\le h_0,
\end{gathered}
    \label{eq:restricted-orbit}
\end{equation}
and a compact coefficient set
$\mathcal X_{\mathrm{orb}}\Subset(\beta_-,\beta_+)^r$.  Assume that,
uniformly on this orbit,
\begin{align}
    c_Vs|\phi-\phi'|
    &\le
    d_{\mathrm{phys}}
    \notag\\
    &\quad
    \left\{
        \vartheta(D_\phi,S_0),
        \vartheta(D_{\phi'},S_0)
    \right\}
    \le
    C_Vs|\phi-\phi'|,
    \label{eq:target-secant}\\
    cs\chi_s|\phi-\phi'|
    &\le
    \inf_{x'\in\mathcal X_{\mathrm{orb}}}
    \notag\\
    &\quad
    \|
        \mu(D_\phi,S_0,x)
        -
        \mu(D_{\phi'},S_0,x')
    \|_2
    \notag\\
    &\le
    Cs\chi_s|\phi-\phi'|.
    \label{eq:test-mean-secant}
\end{align}
The first condition makes physical target displacement locally proportional to
$s|\phi-\phi'|$; the second defines the corresponding profiled deployment
secant scale $\chi_s$.

\subsubsection{Two-atom physical target secant}
\label{sec:supp-two-atom-secant}

The general secant is easiest to interpret for two active coherent atoms.  By
rotating their difference while fixing their sum, equal coefficients cancel
the orientation change exactly, whereas a nonzero coefficient contrast makes
the change observable.  This gives the concrete formula used in the main-text
corollary.

For $S=\{1,2\}$ on the balanced submodel, write
\[
    e=v_1+v_2,
    \qquad
    f=v_1-v_2.
\]
Choose $g\perp\operatorname{span}\{e,f\}$ with $\|g\|=\|f\|$, and a skew
generator $\Omega$ such that
\[
    \Omega e=0,
    \qquad
    \exp(h\Omega)f
    =
    \cos(h)f+\sin(h)g.
\]
Then
\[
    \Omega v_1=\frac12g,
    \qquad
    \Omega v_2=-\frac12g.
\]
For the coherent atoms along this orbit, a sufficiently short compact chart
satisfies
\begin{equation}
\begin{aligned}
    c_1s|h-h'|
    &\le
    d_{\mathrm{pr}}
    (
        [d_j(h)],
        [d_j(h')]
    )
    \\
    &\le
    C_1s|h-h'|,
    \qquad
    j=1,2.
\end{aligned}
    \label{eq:sm-tech-s6-2}
\end{equation}
Choosing the chart short enough that cross-label matches remain separated
makes the same-label matching Hausdorff-optimal, so
\begin{equation}
\begin{aligned}
    c_Vs|h-h'|
    &\le
    d_H^{\mathrm{pr}}
    \left(
        \{[d_1(h)],[d_2(h)]\},
        \{[d_1(h')],[d_2(h')]\}
    \right)
    \\
    &\le
    C_Vs|h-h'|.
\end{aligned}
    \label{eq:sm-tech-s6-3}
\end{equation}

For coefficients
\[
    x_1=\bar\beta+\frac d2,
    \qquad
    x_2=\bar\beta-\frac d2,
\]
exact coefficient profiling yields
\begin{equation}
    \inf_{\bar\beta',d'\in\mathbb R}
    \|
        \mu_{h,\bar\beta,d}
        -
        \mu_{h',\bar\beta',d'}
    \|_2
    =
    \frac{
        \lambda_\star s|d|\|f\|_2
    }2
    |\sin(h-h')|.
    \label{eq:sm-tech-s6-4}
\end{equation}
On any compact sign-preserving coefficient set with
$|d|\asymp|d_0|>0$, this verifies the two-sided deployment secant used in
Theorem~\ref{thm:task-symmetry}.
On the separate exact slice $d=0$, the deployment mean is invariant along the
orientation orbit.

\subsubsection{Proof of the deployment-information results}
\label{sec:supp-proof-task}

We now combine coefficient profiling with the Gaussian deployment likelihood.
The projected tangent gives the efficient deployment information, calibration and deployment KL contributions add by independence on the restricted orbit,
and the two-point diameter bound gives the matching lower rate.

\begin{proof}[Proof of Theorem~\ref{thm:task-symmetry} and
Corollary~\ref{cor:two-atom-contrast}]

Let
\[
    w_{S,x}
    =
    \sum_{j\in S}x_jv_j,
    \qquad
    \mathcal V_S^0
    =
    \left\{
        \sum_{j\in S}c_jv_j:
        \sum_{j\in S}c_j=0
    \right\}.
\]
At the centered balanced submodel, an orientation tangent changes the
deployment mean by
\[
    \dot\mu_R
    =
    \lambda_\star sR\Omega w_{S,x}.
\]
Coefficient perturbations with zero sum span
$\lambda_\star sR\mathcal V_S^0$, while a nonzero coefficient sum creates an
axial component and cannot cancel a pure order-$s$ transverse orientation
change.  Orthogonal projection therefore gives
Equation~\eqref{eq:tangent-secant}, and the Gaussian likelihood gives
Equation~\eqref{eq:test-efficient-information}.

For the restricted orientation orbit, two retained parameters separated by
$h=\phi-\phi'$ must satisfy
\[
    s^3|h|\lesssim N^{-1/2}
\]
from the calibration profile and
\[
    s\chi_s|h|\lesssim \sigma T^{-1/2}
\]
from the deployment profile.  The target secant therefore gives the upper
rate in Equation~\eqref{eq:restricted-task-rate}.  For the lower bound, we choose
\[
    |h|
    \asymp
    \min\left\{
        h_0,
        \left(
            Ns^6+T\chi_s^2s^2/\sigma^2
        \right)^{-1/2}
    \right\}.
\]
Calibration and deployment KL divergences add by independence, while the
target separation is at least $cs|h|$.
The two-point diameter bound in
Equation~\eqref{eq:supp-two-point-diameter} gives the matching lower rate.

For $r=2$, the technical secant calculation above gives
\[
    \inf_{\bar\beta',d'\in\mathbb R}
    \|
        \mu_{\phi,\bar\beta,d}
        -
        \mu_{\phi',\bar\beta',d'}
    \|_2
    =
    \frac{
        \lambda_\star s|d|\|v_1-v_2\|_2
    }{2}
    |\sin(\phi-\phi')|.
\]
On the declared sign-preserving coefficient set,
$|d|\asymp|d_0|$, which verifies the two-sided deployment secant and yields
Equation~\eqref{eq:amplitude-crossover}.  When $d=0$, the deployment mean is exactly
invariant along the orientation path, so deployment replication provides no
orientation information.
\end{proof}

%% file: appendices/B_finite_bank_and_numerical_details.tex
\section{Finite-Bank Implementation and Numerical Details}
\label{sec:supp-computation}

This section records the implementation details needed to reproduce the
numerical studies in Section~\ref{sec:experimental-evaluation}.  The first
part isolates the two theoretical mechanisms tested numerically: sixth-order
calibration sensitivity and coefficient-dependent deployment information.  The
second part specifies how a finite-bank candidate is evaluated and why the
held-out likelihood-ratio rule retains an on-bank truth candidate.  The final
parts describe the global and four-region studies, including the bank
construction, data splits, query policies, denominators, and query accounting.

These calculations have deliberately different roles.  The theory-guided
diagnostic illustrates a local mechanism and does not numerically implement
the continuous confidence correspondence.  The AEB experiments compare
adaptive evaluation with exhaustive evaluation of the same specified finite
bank; the banks are not assumed to approximate or outer-cover the continuous
parameter space.

\subsection{Theory-guided numerical diagnostics}
\label{sec:supp-numerical}

The purpose of this calculation is not to test an optimization algorithm.
Instead, it asks whether the two local mechanisms derived in
Section~\ref{sec:results} are visible directly in the controlled probability
model: does calibration distinguish orientation only at sixth order, and does
deployment become informative only when the coefficient profile exposes the
orientation change?

\subsubsection{Balanced model and exact calibration mixture}

We use the balanced $q=4$ construction because it keeps the local geometry
exactly controlled while still allowing a nontrivial orientation path.  The
resulting Bernoulli--Gaussian calibration law is a finite Gaussian mixture, so
the distributional scaling can be evaluated without fitting a dictionary.

Let $v_1,\ldots,v_4\in\mathbb R^3$ be the tetrahedral vertices
\[
\begin{gathered}
3^{-1/2}(1,1,1),\quad
3^{-1/2}(1,-1,-1),\\
3^{-1/2}(-1,1,-1),\quad
3^{-1/2}(-1,-1,1).
\end{gathered}
\]
With $u=e_4$, define the coherent atoms along the orientation path by
\begin{equation}
d_j(s,\phi)
=
\sqrt{1-s^2}\,u+sR_\phi v_j,
\label{eq:supp-numerical-path}
\end{equation}
where $R_\phi$ rotates about $v_1+v_2$.
Thus $R_\phi(v_1+v_2)=v_1+v_2$, and the deployment mean on
$S=\{1,2\}$ with equal coefficients is invariant.
The separated atom is $(u+e_1)/\sqrt2$.

We use
\[
\begin{gathered}
s\in\{0.045,0.055,0.067,0.082,0.100,0.122\},
\\
p=0.2,
\qquad
\nu=0.5,
\end{gathered}
\]
unit active-code variance, and path parameter $h=0.25$.
The skew generator is normalized by
\[
    \max_{j\in S}\|\Omega v_j\|=1.
\]

For a support mask $a\in\{0,1\}^5$, let
\[
w_a
=
p^{|a|}(1-p)^{5-|a|},
\qquad
\Sigma_a(\phi)
=
\nu I_4
+
D_\phi\operatorname{diag}(a)D_\phi^\top.
\]
The one-observation calibration density is evaluated as the exact
32-component Gaussian mixture
\begin{equation}
f_{s,\phi}(y)
=
\sum_{a\in\{0,1\}^5}
w_a\,
\varphi_4\{y;0,\Sigma_a(\phi)\}.
\label{eq:supp-exact-mixture}
\end{equation}
Hence this diagnostic uses neither a dictionary-learning optimizer nor a
numerical approximation to the continuous confidence region.

\subsubsection{Stable divergence calculation}

In the highly coherent regime, the two orientation distributions are extremely
close, so direct subtraction of likelihood-based divergences is numerically
fragile.  We therefore use midpoint identities that remain nonnegative and
stable near equality, and we verify the resulting power law both by Monte
Carlo and by deterministic Gauss--Hermite quadrature.

Let
\[
\begin{gathered}
P=f_{s,0},
\qquad
Q=f_{s,\phi},
\\
M=\frac{P+Q}{2},
\qquad
L=\log\frac{dP}{dQ}.
\end{gathered}
\]
To avoid cancellation in the near-coherent regime, we use the nonnegative
midpoint identities
\begin{align}
J(P,Q)
&=
2\,\mathbb E_M\{L\tanh(L/2)\},
\label{eq:supp-jeffreys-midpoint}\\
1-A(P,Q)
&=
\mathbb E_M
\left[
\frac{\tanh^2(L/2)}
{1+\operatorname{sech}(L/2)}
\right],
\label{eq:supp-affinity-midpoint}
\end{align}
where $J$ is Jeffreys divergence and
$A(P,Q)=\int\sqrt{dP\,dQ}$ is Hellinger affinity.

At each value of $s$, 128 common-random-number batches of size 8192 give
$2^{20}$ midpoint observations.
Paired resampling across the six scales produces the reported slope interval.
As an independent deterministic check, tensor Gauss--Hermite rules of orders
10 and 14 integrate every Gaussian component in
Equation~\eqref{eq:supp-exact-mixture}; the order-14 all-scale fit has exponent
$5.9342$ and $R^2=0.99999$.

For $N$ independent calibration observations,
\begin{equation}
A(P^{\otimes N},Q^{\otimes N})
=
A(P,Q)^N.
\label{eq:supp-product-affinity}
\end{equation}
For two Gaussian deployment means with common covariance
$\sigma^2I/T$, the affinity is
\begin{equation}
\exp\left\{
-\frac{T\|\mu_0-\mu_1\|_2^2}{8\sigma^2}
\right\}.
\label{eq:supp-gaussian-affinity}
\end{equation}
These identities are used only to visualize the separate calibration,
support-discrimination, and deployment-model information scales.

\ifdefined\TSPPUBLICATION
\begin{table*}[t]
\else
\begin{table}[htbp]
\fi
\caption{Theory-guided numerical diagnostics.  Monte Carlo intervals use the
128 paired batches; quadrature and deployment checks are deterministic.}
\label{tab:supp-numerical-summary}
\centering
\small
\begin{tabularx}{\linewidth}{@{}
  >{\raggedright\arraybackslash}p{.27\linewidth}
  >{\raggedright\arraybackslash}p{.35\linewidth}
  >{\raggedright\arraybackslash}X@{}}
\toprule
Diagnostic & Result & Interpretation \\
\midrule
Jeffreys $s$-exponent
& 5.9355; central 95\% paired-batch MC range [5.9249, 5.9468]
& sixth-order scaling; numerical stability \\
Jeffreys log--log $R^2$
& 0.99999
& stable scaling \\
Order-14 quadrature exponent
& 5.9342
& independent check \\
Maximum cross-$s$ affinity spread
& 0.0243
& approximate collapse \\
Median transition-band spread
& 0.0167
& finite-$s$ drift \\
Prespecified spread tolerance
& 0.015
& not met \\
Equal-coefficient deployment residual
& $6.82\times10^{-16}$
& invariant to numerical precision \\
Contrast-formula relative error
& $7.01\times10^{-15}$
& analytical agreement \\
\bottomrule
\end{tabularx}
\ifdefined\TSPPUBLICATION
\end{table*}
\else
\end{table}
\fi

At each fixed separation scale, 128 common-random-number batches of 8192
midpoint observations---4096 from each endpoint mixture---estimate the
Jeffreys divergence, for $2^{20}$ observations per scale.  The same batch
index is paired across all six fixed scales.  For each of 2000 bootstrap
replicates (seed 2026072202), 128 batch indices are sampled with replacement,
the batch estimates are averaged within each scale, and the ordinary
least-squares slope of log mean divergence against log separation scale is
recomputed.  The reported central 95\% range is the 2.5th and 97.5th
percentiles of those slopes.  It is a Monte Carlo stability diagnostic
conditional on the fixed mixture model and grid, not uncertainty across
independent datasets.  The underlying midpoint Monte Carlo seed is 2026072201;
the per-scale error bars are separate Student-$t$ intervals over the 128 batch
means.

\begin{figure}[htbp]
\centering
\includegraphics[width=.82\linewidth]{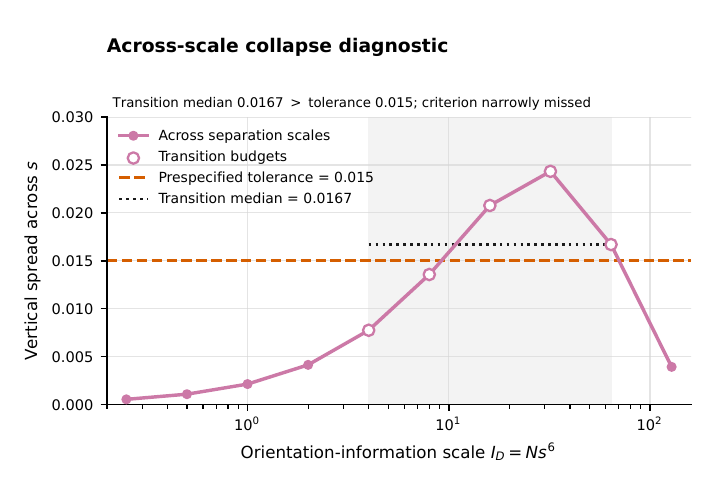}
\caption{Across-scale orientation-information collapse diagnostic underlying
Figure~\ref{fig:theorem-native-mechanism}.  At each stored
$I_D=Ns^6$, the plotted vertical spread is the maximum minus the minimum
product affinity across the saved separation scales.  Over the prespecified
transition budgets $I_D\in\{4,8,16,32,64\}$, the median spread is $0.0167$,
which narrowly exceeds the prespecified tolerance $0.015$.  The criterion is
therefore recorded as missed and is not used as evidence of numerical coverage
for the continuous confidence correspondence.}
\label{fig:supp-collapse-diagnostic}
\end{figure}

\subsubsection{Deployment task control and robustness checks}

The deployment calculation is designed to separate the effect of replication
from the effect of the coefficient profile.  Equal coefficients provide an
exact invariance control: the orientation changes physically, but the
deployment mean can remain unchanged after coefficient profiling.  Nonzero
contrasts provide the complementary informative cases predicted by
Corollary~\ref{cor:two-atom-contrast}.

\begin{figure}[htbp]
\centering
\includegraphics[width=.68\linewidth]{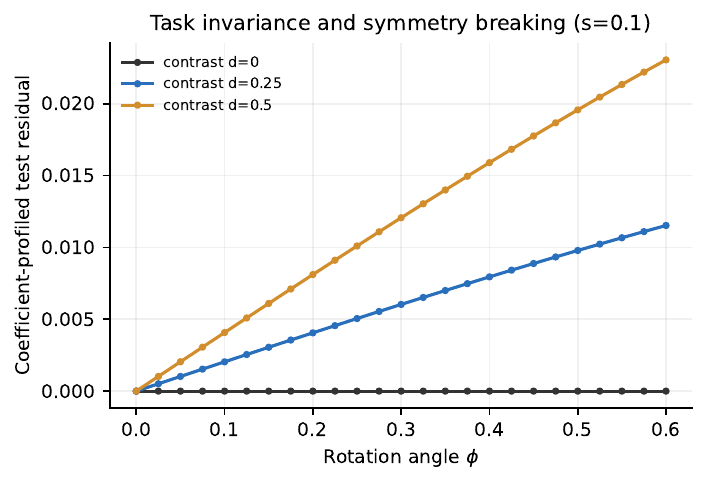}
\caption{Coefficient-profiled deployment residual at $s=0.1$.
Equal coefficients are invariant to numerical precision, whereas nonzero
coefficient contrasts follow the closed form in
Equation~\eqref{eq:two-atom-profile}.}
\label{fig:supp-task-invariance}
\end{figure}

For contrasts $d\in\{0.25,0.5\}$, the numerical least-squares residual is
compared with
\[
\frac{s|d|\|v_1-v_2\|_2}{2}|\sin\phi|.
\]
The largest relative discrepancy over the nonzero displayed cells is
$7.01\times10^{-15}$.
Optimal projective matching preserves the same coherent-atom labels throughout
the displayed path.

Additional robustness checks were performed after the primary calculation.
The fitted $h$-exponent is $1.943$ with $R^2=0.99983$.
Across the nine cells
\[
p\in\{0.1,0.2,0.35\},
\qquad
\nu\in\{0.3,0.5,1.0\},
\]
every log--log $R^2$ exceeds $0.9999$.
The smallest fitted $s$-exponent is $5.759$, occurring at the smallest
activation probability and noise level.
These calculations are reported as robustness diagnostics, not as a new
fitted theory.

\subsection{Finite-bank evaluator and certificate implementation}
\label{sec:supp-aeb-protocol}

This subsection records the implementation corresponding to the finite-bank
construction in Section~\ref{sec:finite-bank-controller}.  The finite bank
\[
    \mathcal B=\{e_1,\ldots,e_M\}
\]
and the requested report family are fixed before held-out candidate
evaluation.  Unqueried or numerically indeterminate candidates remain
possible explanations throughout the adaptive run.

For the represented-scale absence fields used in the application, if
$\lambda_R(e)$ is the source strength represented by candidate $e$ in region
$R$, the full predicate is
\begin{equation}
    q_{\mathrm{abs}(R,b_R)}(e)
    =
    \mathbf 1\{\lambda_R(e)<b_R\}.
    \label{eq:thresholded-absence-predicate}
\end{equation}
This assertion is relative to the candidate family and source-strength range
represented in the finite bank.

For candidate $e$, let $\ell_t(e)$ denote its log likelihood-ratio e-value at
checkpoint $t$, and define
\begin{equation}
    \tau_{\alpha_{\mathcal B}}
    =
    -\log\alpha_{\mathcal B},
    \qquad
    m(e)
    =
    \max_t\ell_t(e)-\tau_{\alpha_{\mathcal B}} .
    \label{eq:supp-log-e-margin}
\end{equation}
After the required numerical evaluation, $m(e)\le0$ is classified
\textsc{admissible} and $m(e)>0$ is classified \textsc{rejected}.
A queried candidate whose sign cannot be certified is
\textsc{indeterminate}; it remains in the possible set.

After $k$ logical queries, the implementation maintains
\begin{equation}
\begin{aligned}
    L_k
    &=
    \{
        e\in\mathcal B:
        e\text{ has been queried and is admissible}
    \},\\
    U_k
    &=
    \{
        e\in\mathcal B:
        e\text{ has not been rejected}
    \}.
\end{aligned}
\label{eq:supp-aeb-lower-upper}
\end{equation}
If $\mathcal A_{\mathcal B}$ is the exhaustive profile obtained by applying
the same classification rule to the full bank, then
\begin{equation}
    L_k
    \subseteq
    \mathcal A_{\mathcal B}
    \subseteq
    U_k.
    \label{eq:supp-aeb-sandwich}
\end{equation}

The deterministic certificate rule
\[
    \mathsf{Cert}_{\mathcal G}(L,U)
    \in
    \mathcal O_{\mathcal G}\cup\{\bot\}
\]
implements the universal and witness conditions of
Section~\ref{sec:finite-bank-controller}.  It returns a report only when
$U\ne\varnothing$, every required truth-level predicate holds for all
candidates in $U$, and every required ambiguity field has its prescribed
opposing admissible witnesses in $L$; otherwise it returns $\bot$.  The
predeclared priority function
$\pi_{\mathcal G}(e; L, U)$ and deterministic tie-breaking rule order the
unqueried candidates that can still change an unresolved certificate.  The
global and four-region policies below are concrete instances of these generic
rules.

Empty profile is a separate terminal state.  It is returned only when
$U_k=\varnothing$, equivalently, when every bank candidate has been certified
rejected.  A query-limited state with no admissible witness, or a fully queried
state containing an indeterminate candidate, is not interpreted as empty; it
leads to \textsc{abstain} unless an independently valid truth-level field has
already been certified.

Algorithm~\ref{alg:supp-aeb} gives the complete controller, including the early
empty-profile check, terminal safeguard, and abstention return.

\begin{algorithm}[H]
\caption{Active endpoint bracketing on a finite explanation bank}
\label{alg:supp-aeb}
\begin{algorithmic}[1]

\REQUIRE
Finite bank $\mathcal B$, held-out evaluator $\psi_W$, requested report family
$\mathcal G$, certificate rule $\mathsf{Cert}_{\mathcal G}$, priority rule
$\pi_{\mathcal G}$, and query budget $K$

\STATE
$L\leftarrow\varnothing$, $U\leftarrow\mathcal B$

\STATE
$Q_{\max}\leftarrow\min(K,M)$

\FOR{$k=0,1,\ldots,Q_{\max}$}

    \IF{$U=\varnothing$}
        \STATE
        \textbf{return} an empty-profile status
    \ENDIF

    \STATE
    $o\leftarrow\mathsf{Cert}_{\mathcal G}(L,U)$

    \IF{$o\ne\bot$}
        \STATE
        \textbf{return} $o$ together with $(L,U)$
    \ENDIF

    \IF{$k<Q_{\max}$}

        \STATE
        Select the next candidate using the predeclared tie-breaking rule:
        \[
            e
            \in
            \argmin_{\substack{
                e'\in U\\
                e'\text{ unqueried}
            }}
            \pi_{\mathcal G}(e';L,U)
        \]

        \STATE
        Evaluate $\psi_W(e)$

        \IF{$\psi_W(e)=\textsc{admissible}$}
            \STATE
            $L\leftarrow L\cup\{e\}$
        \ELSIF{$\psi_W(e)=\textsc{rejected}$}
            \STATE
            $U\leftarrow U\setminus\{e\}$
        \ELSE
            \STATE
            retain $e$ in $U$
        \ENDIF

    \ENDIF

\ENDFOR

\IF{$U=\varnothing$}
        \STATE
        \textbf{return} an empty-profile status
    \ELSE
        \STATE
        \textbf{return} \textsc{abstain} together with $(L,U)$ and the
        incomplete certificate state
    \ENDIF

\end{algorithmic}
\end{algorithm}

\subsubsection{Conditional candidate retention}

The adaptive priority rule may depend on proposal data, but the held-out
observations used to reject a candidate must remain untouched by that proposal
construction.  Conditioning on the proposal stage therefore turns the
likelihood-ratio process into an ordinary mean-one martingale under the
candidate being tested, which is exactly the structure needed for Ville's
inequality.

Let $\mathscr H_{\mathrm{prop}}$ contain the data split, all proposal-stage
observations, and the proposal and policy objects constructed from them.
Conditional on $\mathscr H_{\mathrm{prop}}$, the proposal density is fixed
before any held-out evaluation observation is used.  For the generic
one-stream evaluator, let $W=(W_1,\ldots, W_m)$ denote the held-out evaluation
sample.  For a fully specified on-bank candidate $e$, assume
$f_{\widehat e}\ll f_e$.  Then
\begin{equation}
    E_t(e)
    =
    \prod_{i\le t}
    \frac{f_{\widehat e}(W_i)}{f_e(W_i)}.
    \label{eq:supp-one-stream-eprocess}
\end{equation}
Under $\mathbb P_e$,
\[
\mathbb E_e
\left[
\frac{f_{\widehat e}(W_t)}{f_e(W_t)}
\middle|
\mathcal F_{t-1},\mathscr H_{\mathrm{prop}}
\right]
=
1,
\]
so $E_t(e)$ is a nonnegative conditional martingale.  The proposal density
may lie outside the finite bank; only its conditional independence from the
held-out evaluation sample is required.

The four-region evaluator specializes $W$ into independent calibration and
deployment streams $Y$ and $Z$.  For candidate $e=(g,S)$,
\begin{equation}
    E_k(e)
    =
    \prod_{i\le n_k}
    \frac{f^C_{\widehat g}(Y_i)}{f^C_g(Y_i)}
    \prod_{j\le t_k}
    \frac{f^D_{\widehat e}(Z_j)}{f^D_e(Z_j)}
    \label{eq:supp-two-stream-eprocess}
\end{equation}
along the deterministic checkpoint path
\begin{equation}
\begin{gathered}
(n_k,t_k)
\in
\{
(128,0),(256,0),(512,0),\\
(1024,0),(2048,0),(2253,116)
\}.
\end{gathered}
\label{eq:supp-six-checkpoints}
\end{equation}
Conditional Ville's inequality therefore gives candidate-wise rejection
probability at most $\alpha_{\mathcal B}$ when rejection requires a strict
crossing of $1/\alpha_{\mathcal B}$.  This protects the represented true
candidate; it is not a simultaneous nonrejection guarantee for all candidates.

The implementation evaluates likelihood-ratio scores in binary64 and
re-evaluates designated near-threshold or certificate-critical candidates at
90-decimal precision.  This high-precision check is not directed interval
arithmetic.  The formal probability statement is therefore the exact-real
candidate-wise result above; the numerical implementation is documented to
make the finite candidate classifications reproducible.

\subsubsection{Proofs of the finite-bank guarantees}
\label{sec:supp-aeb-guarantees}

The three finite-bank guarantees have different roles.  The first is
statistical and controls rejection of the represented true candidate.  The
second is deterministic and relates the partial $(L_k, U_k)$ state to
exhaustive same-bank evaluation.  The third combines the two to control false
truth-level reports.

\begin{proof}[Proof of Corollary~\ref{cor:eprocess-retention}]
Conditional on $\mathscr H_{\mathrm{prop}}$, the process $E_t(e)$ is a
nonnegative mean-one martingale under candidate $e$.  Ville's inequality and
the predeclared checkpoint set therefore give
\[
    \mathbb P_e
    \left\{
        \max_{t\in\mathcal T}E_t(e)
        >
        \frac{1}{\alpha_{\mathcal B}}
        \mathrel{\Big|}
        \mathscr H_{\mathrm{prop}}
    \right\}
    \le
    \alpha_{\mathcal B}.
\]
Integrating over the proposal-stage data proves the candidate-wise rejection
bound in Equation~\eqref{eq:on-bank-truth-retention}.
\end{proof}

\begin{proof}[Proof of Theorem~\ref{thm:finite-bank-bridge}]
The inclusion
$\mathcal A_{\mathcal B}(W)\subseteq U_k$ implies that every assertion holding
pointwise throughout $U_k$ also holds throughout the exhaustive finite-bank
profile.  Similarly, $L_k\subseteq\mathcal A_{\mathcal B}(W)$ implies that
opposing witnesses in $L_k$ both belong to the exhaustive profile.  Finally,
$F\subseteq C$ gives
$q_{\mathrm{fine}(R,F)}(e)\le q_{\mathrm{sector}(R,C)}(e)$ for every candidate
$e$.  These are pathwise statements at each query prefix, so they remain valid
at a data-dependent stopping prefix.
\end{proof}

\begin{proof}[Proof of Theorem~\ref{thm:on-bank-truth-validity}]
Fix $e_\star\in\mathcal B$.  If a reported truth-level assertion is false at
$e_\star$, then $q_{\widehat g}(e_\star)=0$.  Because every reported assertion
holds throughout the possible set, this event implies
$e_\star\notin U_k$.  Candidates leave $U_k$ only when certified rejected;
hence
\[
\begin{aligned}
    \mathbb P_{e_\star}
    \left\{
        \begin{gathered}
        \widehat g\text{ is reported and}\\
        q_{\widehat g}(e_\star)=0
        \end{gathered}
    \right\}
    &\le
    \mathbb P_{e_\star}
    \left\{
        e_\star\text{ is ever rejected}
    \right\}
    \\
    &\le
    \alpha_{\mathcal B}.
\end{aligned}
\]
Taking the supremum over $e_\star\in\mathcal B$ proves
Equation~\eqref{eq:on-bank-false-assertion-bound}.
\end{proof}

\subsection{Global finite-bank evaluation protocol}
\label{sec:supp-global-protocol}

The global study is designed to evaluate the computation--resolution tradeoff
of AEB independently of the application-specific regional predicates used in
the four-region example.  We vary the information regime and the requested
report profile so that query behavior is tested under both easy and difficult
certificate obligations.

The global study uses
\[
q=4,\qquad n=5,\qquad
\phi_{\rm truth}=0,\qquad
p_{\rm truth}=0.20,
\]
32 calibration-mixture components and active support $\{0,1\}$.
Each dataset is split into proposal and held-out evaluation fractions
$0.45$ and $0.55$.
The global study fixes its requested report family, certificate rule, priority
function, and deterministic tie-breaking order before held-out evaluation.

The proposal orientation $\widehat\phi_A$ is retained as a fixed reference in
the global diameter calculation.
If $S$ denotes the exhaustive set of surviving bank orientations, define
\[
\begin{aligned}
W_k^+
&=
\{\widehat\phi_A\}
\cup
\{
\substack{
\phi_j:\\
\phi_j\text{ has a queried admissible}\\
\text{nuisance witness}}
\},\\
P_k^+
&=
\{\widehat\phi_A\}
\cup
\{
\substack{
\phi_j:\\
\phi_j\text{ has not been eliminated}}
\},\\
S^+
&=
\{\widehat\phi_A\}\cup S .
\end{aligned}
\]
Correct candidate classifications give
\[
\begin{aligned}
    W_k^+
    \subseteq
    S^+
    \subseteq
    P_k^+,
    \\
    \operatorname{diam}(W_k^+)
    \le
    \operatorname{diam}(S^+)
    \le
    \operatorname{diam}(P_k^+).
\end{aligned}
\]
Thus, this experiment evaluates a proposal-referenced global physical-resolution
summary; it does not use the regional support predicates of the four-region
experiment.

\ifdefined\TSPPUBLICATION
\begin{table*}[t]
\else
\begin{table}[htbp]
\fi
\centering
\caption{Global finite-bank evaluation conditions.}
\label{supp-tab:global-protocol}
\small
\begin{tabular}{lrrrrr}
\toprule
Condition & $N$ & $s$ & $\nu_{\rm truth}$ & Bank size & Cases \\
\midrule
Low information
& 4096 & 0.35 & 1.20 & 1025 & 6 \\
Intermediate information
& 65536 & 0.50 & 0.90 & 1025 & 6 \\
High information
& 131072 & 0.50 & 0.80 & 369 & 6 \\
\bottomrule
\end{tabular}
\ifdefined\TSPPUBLICATION
\end{table*}
\else
\end{table}
\fi

The data seeds are
$2026080111$--$2026080116$,
$2026080121$--$2026080126$, and
$2026080131$--$2026080136$ for the three conditions, respectively.
The corresponding split seeds are obtained by replacing the leading $2$ with
$9$.

Each dataset is evaluated under three independently specified reporting
profiles:
\[
\begin{array}{c|ccc}
 & \alpha_{\mathcal B} & \delta_f/d_{\rm shell} & \delta_s/d_{\rm shell}\\ \hline
\textsc{risk-conservative}
&0.025&0.40&0.50\\
\textsc{balanced}
&0.077&0.35&0.60\\
\textsc{resolution-favoring}
&0.150&0.25&0.40
\end{array}
\]
giving 54 controller traces from 18 independent datasets.
Results are recorded at query-budget fractions
\[
    \{0.10,0.20,0.35,0.50,0.75,1.00\}.
\]
A logical query is counted whenever the controller requests a candidate
classification, even if a numerical score is already cached.

After every controller trace has been fixed, the same finite bank is evaluated
exhaustively using the same candidate-classification rule.
The finite-bank reference therefore does not influence the adaptive query
order.
Near-threshold candidates and candidates affecting terminal or elimination
certificates receive the high-precision re-evaluation described above.

The principal denominators in the main-text summary are 54 controller traces,
34 exhaustive-bank ambiguity traces, seven exhaustive-bank fine traces, and
20 exhaustive-bank nonambiguous traces.
Full-budget categorical agreement is evaluated over all 54 traces.
The 54 traces arise from three reporting profiles applied to 18 independent
datasets and must not be interpreted as 54 independent datasets.

Table~\ref{tab:global-controller} gives the complete summary at the two
displayed operating points, and Figure~\ref{fig:supp-global-recovery} separates
ambiguity and fine recovery at those budgets.

\input{tables/global_finite_bank_summary}

\begin{figure}[htbp]
\centering
\includegraphics[width=\linewidth]
{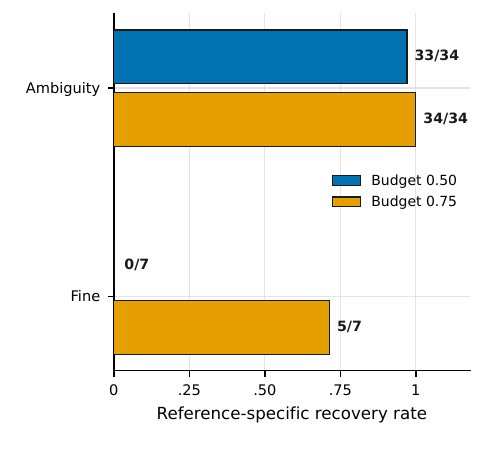}
\caption{Global AEB recovery at query budgets $0.50$ and $0.75$. The 54
reporting traces arise from three predeclared profiles applied to 18
independent cases; the recovery denominators are 34 exhaustive-bank ambiguity
traces and seven exhaustive-bank fine traces.}
\label{fig:supp-global-recovery}
\end{figure}

The median queried bank fraction is $0.209$ at both budget 0.50 and budget
0.75.  This is the median of normalized trace-specific query fractions, not
raw query counts.  Thirty-four terminal fractions are already at or below
0.50; the 27th and 28th order statistics are $201/1025$ and $227/1025$, whose
average is $214/1025=0.208780\ldots$.  These middle observations are therefore
fixed before either displayed cap.  Raising the cap changes the upper tail and
fine recovery rather than the median stopping fraction.

\subsection{Four-region application protocol}
\label{sec:supp-four-region-protocol}

The four-region study is designed to place several kinds of physical
conclusion in one controlled bank: fine localization in region A, group-level
resolution in the coherent region B, presence--absence ambiguity for weak
region C, and represented-scale absence with D-present controls in region D.
This lets the same AEB logic be evaluated across qualitatively different
certificate obligations.

The four-region response library contains 12 atoms, 72 dictionary states, and
the support patterns AB, ABC, and ABD, giving
\[
    M=216
\]
complete candidate explanations.
Figure~\ref{fig:supp-four-region-bank} records the response library and its
absolute atom-coherence structure.

\ifdefined\TSPPUBLICATION
\begin{figure*}[t]
\else
\begin{figure}[htbp]
\fi
  \centering
  \subfloat[Response library. The 12 atoms are grouped into physical regions
  A--D; color identifies the region, and line style or marker identifies the
  atom within a region.\label{fig:supp-four-region-bank:a}]{%
    \includegraphics[width=.535\linewidth]
    {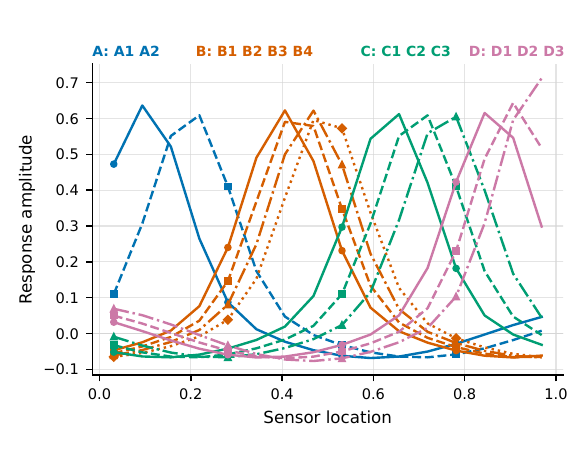}%
  }
  \hfill
  \subfloat[Absolute atom coherence. Region blocks expose the isolated,
  highly coherent, weak optional, and optional-interferer structures used in
  the application.\label{fig:supp-four-region-bank:b}]{%
    \includegraphics[width=.425\linewidth]
    {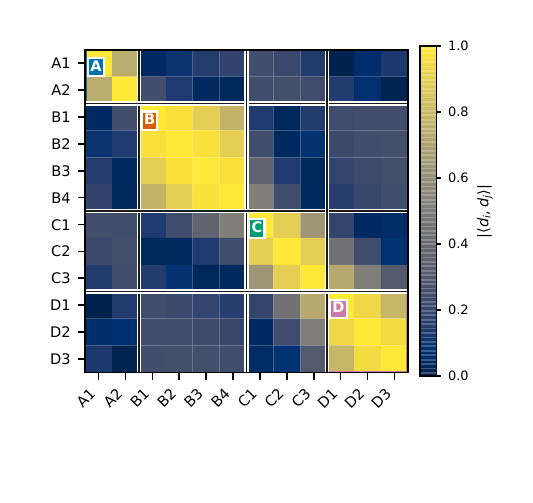}%
  }
  \caption{Four-region finite-bank construction used for the application
  protocol. The library contains 12 atoms arranged into the four physical
  regions, and the coherence matrix shows the within- and between-region atom
  similarities that determine the available physical resolution.}
  \label{fig:supp-four-region-bank}
\ifdefined\TSPPUBLICATION
\end{figure*}
\else
\end{figure}
\fi

The four-region study fixes its requested report family, certificate rule,
priority function, and deterministic tie-breaking order before held-out
evaluation.
The fixed design parameters are
\[
\begin{gathered}
    \alpha_{\mathcal B}=0.077,
    \qquad
    h=0.085,
    \\
    \tau_A=1,
    \qquad
    \tau_B=0.8,
    \\
    \tau_C=0.1,
    \qquad
    \tau_D=1.00.
\end{gathered}
\]
For candidate $e=(g,S)$,
\begin{equation}
    \Sigma_e
    =
    \nu_D I
    +
    \sum_{R\in S(e)}
    \tau_R^2a_R(g)a_R(g)^\top .
    \label{eq:supp-four-region-law}
\end{equation}
The represented regional strength is
\[
    \lambda_R(e)
    =
    \tau_R\mathbf1\{R\in S(e)\}.
\]
Accordingly, a reported absence statement excludes the optional-region-present
candidate family at the represented scale $\tau_R$.
It does not assert absence over a continuum of amplitudes or arbitrarily weak
off-bank sources.

Each dataset contains a $4096\times16$ calibration array and a
$192\times16$ deployment array.
The proposal/evaluation partition contains 1843/2253 calibration observations
and 76/116 deployment observations.
The six evaluation checkpoints are exactly those in
Equation~\eqref{eq:supp-six-checkpoints}.

The 15 fresh datasets consist of six persistent-only cases with seeds
$2026091001$--$2026091006$, six weak-C cases with seeds
$2026091011$--$2026091016$, and three D-present controls with seeds
$2026091021$--$2026091023$.

The point-valued comparator is a proposal-split deployment
maximum-likelihood selector that chooses one dictionary--support explanation
and then reports the physical interpretation conditional on that choice.
Among the 216 explanations, it maximizes the summed fixed-covariance Gaussian
log-likelihood on the 76-observation proposal subset of the 192 deployment
replicates, with the lowest canonical index breaking ties. It then uses the
maximizing explanation's dictionary state and support, reports selected
regions at their fine atom locations, and declares unselected regions absent
above the represented beta-min scale.

The regional policy alternates between actions that can establish opposing support witnesses and actions that can eliminate candidates, preventing a regional certificate.
Its fixed proposal order is the deployment proposal dictionary, followed by
the distinct calibration proposal dictionary, followed by candidate ID.
For the weak-C ambiguity objective, the fixed priority order is:
(i) obtain a C-absent witness,
(ii) obtain a C-present witness,
(iii) eliminate an alternative A location, and
(iv) eliminate the D-present family.
After each query, the regional rule
$\mathsf{Cert}_{\mathcal G}(L,U)$ is recomputed.
The primary query cap is
\[
    162/216.
\]
The $108/216$ state reported in diagnostics is a prefix of the same controller
trace, not a separate run.

The double-precision rejection threshold is
\[
    -\log\alpha_{\mathcal B}
    =
    -\log(0.077)
    =
    2.563949857.
\]
Provisionally admissible, near-threshold, and certificate-critical candidates
receive the high-precision re-evaluation used by the exhaustive finite-bank
reference.
The exhaustive reference scores all 216 candidates and projects only the
admissible subset to the regional physical map.

Table~\ref{tab:application-results} gives the complete four-region outcome
summary before the detailed denominator and empty-profile discussion below.

\input{tables/application_conclusions}

Fourteen of the 15 datasets yield nonempty exhaustive profiles.
The completed-profile denominators used in the main text are:
11 eligible A-fine cases, five B group/sector cases, five C-ambiguity cases,
and ten represented-scale D-absence cases.
The point-valued plug-in comparison uses the same five eligible weak-C cases for regions B and C.
The three D-present controls provide the denominator for the false-D-absence
check.

One weak-C realization produces an empty exhaustive finite profile.
At the 162-query online cap, its lower set is empty, and 54 candidates remain
possible, so AEB returns \textsc{abstain}; only subsequent exhaustive scoring
shows that no candidate in the finite bank is admissible.
The reported physical map is therefore null.
This realization is excluded from truth-relative utility rates based on
completed profiles, but its 162 queries remain in the computational summary.
It is not interpreted as evidence that all physical regions are absent.

The three D-present controls also reach the primary cap without an admissible
lower-set witness, although their exhaustive profiles are nonempty.
Their universal regional truth-level fields may be inspected separately, but
they are not counted as witnessed global nonabstaining maps.
None receives the represented-scale D-absence label.

\subsection{Query accounting, reproducibility, and scope}
\label{sec:supp-reproducibility}

Because cached scores, high-precision replay, and exhaustive reference scoring
can all occur outside the adaptive controller, the notion of a ``query'' must
be fixed explicitly before interpreting the reported savings.  We count only
controller-requested candidate classifications and separately document the
offline work needed for reproducibility.

A logical query is one controller-requested candidate evaluation.
Cache hits do not remove logical queries.
Offline exhaustive finite-bank scoring, high-precision verification,
completion checks, and report generation are excluded from the adaptive query
count and are not interpreted as deployment latency.

The global and four-region studies use different banks, report functionals,
and query policies; their observations and denominators are not pooled.
The publication package records the candidate laws, data splits, checkpoint
grids, proposal and witness priorities, tie order, report predicates, seeds,
raw numerical summaries, and figure data required to reproduce the reported
experiments.
For the theory-guided calculation, the Monte Carlo seed is 2026072201 and the
paired-bootstrap seed is 2026072202.

The numerical evidence has deliberately limited scope.
The theory-guided experiment illustrates the finite-$s$ manifestation of the
$s^6$ calibration geometry and deployment task symmetry.
The finite-bank experiments evaluate AEB relative to exhaustive evaluation of
the same specified candidate banks.
They do not establish numerical coverage of the continuous confidence
correspondence, completeness of the finite banks relative to the continuous
model, off-bank validity, worst-case sublinear query complexity, or transfer
to arbitrary real-data applications.

%% file: tables/global_finite_bank_summary.tex
\ifdefined\TSPREVIEW
\begin{table}[H]
\else
\begin{table}[t]
\fi
  \centering
  \caption{Global AEB outcomes over 18 independent cases and 54 reporting
  traces, relative to exhaustive scoring of the same finite candidate bank.}
  \label{tab:global-controller}
  \begingroup
  \footnotesize
  \setlength{\tabcolsep}{3pt}
  \renewcommand{\arraystretch}{1.12}
  \begin{tabularx}{\columnwidth}{
    >{\raggedright\arraybackslash}X
    >{\centering\arraybackslash}p{.17\columnwidth}
    >{\centering\arraybackslash}p{.17\columnwidth}}
    \toprule
    Outcome relative to exhaustive same-bank reference &
    Budget $0.50$ &
    Budget $0.75$ \\
    \midrule
    Non-abstaining (substantive) report & $41/54$ & $54/54$ \\
    Ambiguity conclusion recovered & $33/34$ & $34/34$ \\
    Fine conclusion recovered & $0/7$ & $5/7$ \\
    Unsafe finer-than-reference conclusion & $0/54$ & $0/54$ \\
    Median queried bank fraction & $0.209$ & $0.209$ \\
    \bottomrule
  \end{tabularx}
  \endgroup
  \vspace{2pt}

  \begin{minipage}{\linewidth}
    \footnotesize
    \emph{Note:} The 54 reporting traces arise from 18 independent cases under
    three predeclared reporting profiles.  Recovery denominators differ by the
    target exhaustive-bank conclusion: 34 ambiguity traces and seven fine
    traces.  At budget $0.75$, the two fine-reference conclusions not reported
    as fine were safely coarsened.  Unsafe counts are the deterministic
    cross-classification of the saved budget outputs against their exhaustive
    same-bank labels.  The trace-level explanation for the identical displayed
    medians is given in the supplement.
  \end{minipage}
\end{table}

%% file: tables/application_conclusions.tex
\ifdefined\TSPREVIEW
\begin{table}[H]
\else
\begin{table*}[t]
\fi
  \centering
  \caption{Physical conclusions across 15 fresh datasets in the four-region
  synthetic application.  Every reference conclusion is obtained by
  exhaustive evaluation of the same finite 216-explanation bank.}
  \label{tab:application-results}
  \begingroup
  \footnotesize
  \setlength{\tabcolsep}{4pt}
  \renewcommand{\arraystretch}{1.12}
  \begin{tabularx}{\textwidth}{
    >{\raggedright\arraybackslash}p{.18\textwidth}
    >{\raggedright\arraybackslash}p{.29\textwidth}
    >{\raggedright\arraybackslash}p{.17\textwidth}
    >{\raggedright\arraybackslash}X}
    \toprule
    Region / phenomenon &
    Exhaustive finite-bank reference &
    AEB outcome relative to reference &
    Point-valued plug-in selector \\
    \midrule
    A: isolated persistent &
    Fine localization in 11 eligible main profiles &
    Fine: $10/11$ &
    \textemdash \\

    B: highly coherent persistent &
    Group/sector level in five eligible completed weak-C profiles &
    Group/sector: $5/5$ &
    Unsupported fine localization: $5/5$ \\

    C: weak optional &
    Support ambiguity in the same five eligible completed weak-C profiles &
    Support ambiguity: $5/5$ &
    False certainty: $5/5$ \\

    D: represented-scale absence &
    Absence at the represented scale in 10 eligible main profiles &
    Absence: $10/10$ &
    \textemdash \\

    D-present controls &
    D present in all three eligible controls &
    False absence: $0/3$ &
    \textemdash \\
    \bottomrule
  \end{tabularx}
  \endgroup
  \vspace{2pt}

  \begin{minipage}{\textwidth}
    \footnotesize
    \emph{Note:} Fifteen fresh datasets yielded 14 nonempty completed
    exhaustive profiles: $11/12$ main profiles and $3/3$ D-present controls.
    One weak-C dataset had an empty exhaustive finite-bank profile and a null
    physical map. Recovery denominators follow the eligibility rules in the
    text; the B and C point-valued plug-in selector comparisons use the same
    five eligible completed weak-C profiles. The empty-profile dataset is
    excluded from truth-relative utility rates but retained in the query-cost
    summary, whose median AEB cost is $148/216$ queries. A dash indicates that
    no aggregate point-valued plug-in selector comparison was designated for
    that endpoint; the representative main-text panel may still display the
    corresponding case-level point-valued plug-in output.
  \end{minipage}
\ifdefined\TSPREVIEW
\end{table}
\else
\end{table*}
\fi